\documentclass{article}
\usepackage{ijcai22}

\usepackage{times}
\usepackage{soul}
\usepackage{url}
\usepackage[hidelinks]{hyperref}
\usepackage[utf8]{inputenc}
\usepackage[small]{caption}
\usepackage{graphicx}
\usepackage{amsmath}
\usepackage{amsthm}
\usepackage{booktabs}
\usepackage{amsfonts}       
\usepackage{nicefrac}       
\usepackage{microtype}      

\usepackage{mathrsfs}
\usepackage{amssymb}
\usepackage{bm}
\usepackage{bbm}
\usepackage{comment}
\usepackage{multirow}
\usepackage{mathtools}
\usepackage{algorithm}
\usepackage[noend]{algpseudocode}

\usepackage{varwidth}
\usepackage{subfig}
\usepackage{enumitem}

\usepackage{tabularx}
\usepackage{xcolor}

\let\OLDthebibliography\thebibliography
\renewcommand\thebibliography[1]{
  \OLDthebibliography{#1}
  \setlength{\parskip}{0pt}
  \setlength{\itemsep}{0pt}
}

\makeatletter
\newcommand{\multiline}[1]{%
  \begin{tabularx}{\dimexpr\linewidth-\ALG@thistlm}[t]{@{}X@{}}
    #1
  \end{tabularx}
}
\makeatother

\makeatletter
\g@addto@macro\normalsize{%
  \setlength\abovedisplayskip{2.5pt}
  \setlength\belowdisplayskip{2.5pt}
  \setlength\abovedisplayshortskip{2.5pt}
  \setlength\belowdisplayshortskip{2.5pt}
}
\makeatother

\renewcommand{\algorithmicensure}{\textbf{Output: }} 

\algnewcommand{\LeftComment}[1]{\(\triangleright\) \textit{#1}}

\algnewcommand{\IIf}[1]{\State\algorithmicif\ #1\ \algorithmicthen}
\algnewcommand{\EndIIf}{\unskip\ \algorithmicend\ \algorithmicif}
\algnewcommand{\IfThen}[2]{
  \State \algorithmicif\ #1\ \algorithmicthen\ #2}
\algnewcommand{\IfThenElse}[3]{
  \State \algorithmicif\ #1\ \algorithmicthen\ #2\ \algorithmicelse\ #3}
\algnewcommand{\ElseIfThen}[2]{
  \State \algorithmicelse~\algorithmicif\ #1\ \algorithmicthen\ #2}
\algnewcommand{\ElseLine}[1]{
  \State \algorithmicelse\ #1}

\newcommand{\exparm}{\texttt{ExploitArm}}
\newcommand{\update}[1]{\textcolor{black}{#1}}

\newtheorem{theorem}{Theorem}

\newtheorem{corollary}[theorem]{Corollary}

\newtheorem{lemma}{Lemma}

\newcommand{\E}{\mathbb{E}}
\newcommand{\N}{\mathbb{N}}

\newcommand{\1}[1]{\mathbbm{1}{\left\{#1\right\}}}
\renewcommand{\P}{\mathbb{P}}

\newcommand{\hatbm}[1]{\hat{\bm{#1}}}
\newcommand{\Mod}[1]{\ (\mathrm{mod}\ #1)}

\newcommand{\ceil}[1]{\left\lceil#1\right\rceil}
\newcommand{\floor}[1]{\left\lfloor#1\right\rfloor}

\newcommand{\abs}[1]{\left\lvert#1\right\rvert}

\renewcommand{\ge}{\geqslant}
\renewcommand{\geq}{\geqslant}
\renewcommand{\le}{\leqslant}

\DeclareMathOperator*{\argmin}{arg\,min}

\DeclareMathOperator*{\kl}{kl}
\DeclareMathOperator*{\Oracle}{Oracle}
\DeclareMathOperator*{\ERT}{\mathbb{E}[\text{\normalfont Reg}(T)]}

\title{Multi-Player Multi-Armed Bandits with Finite Shareable Resources Arms: Learning Algorithms \& Applications\footnote{
  To appear at the 31st International Joint Conference on Artificial Intelligence (IJCAI), 2022.
}}

\author{
Xuchuang Wang$^1$
\and
Hong Xie$^2$\and
John C.S. Lui$^{1}$
\affiliations
$^1$Department of Computer Science \& Engineering, 
The Chinese University of Hong Kong\\
$^2$College of Computer Science, Chongqing University, China
\emails
\(^1\)\{xcwang, cslui\}@cse.cuhk.edu.hk,
\(^2\)xiehong2018@foxmail.com
}

\begin{document}

\maketitle

\begin{abstract}
  \update{Multi-player multi-armed bandits (MMAB) study how decentralized players cooperatively play
    the same multi-armed bandit so as to maximize their total cumulative rewards.
    Existing MMAB models mostly assume when more than one player pulls the same arm,
    they either have a collision and obtain zero rewards,
    or have no collision and gain independent rewards,
    both of which are
    usually too restrictive in practical scenarios.}
  In this paper, we propose an MMAB with \emph{shareable resources} as
  an \update{extension} to
  the collision and non-collision settings.
  Each shareable arm has finite shareable resources and a ``per-load'' reward random variable,
  both of which are unknown to players.
  The reward from a shareable arm is equal to the ``per-load'' reward multiplied by the minimum between the number of players pulling the arm and the arm's maximal shareable resources.
  We consider two types of feedback: sharing demand information (SDI) and sharing demand awareness (SDA),
  each of which provides different signals of resource sharing.
  We design the \texttt{DPE-SDI} and \texttt{SIC-SDA} algorithms to address the shareable arm problem under these two cases of feedback respectively
  and prove that both algorithms have logarithmic regrets
  \update{that are tight in the number of rounds}.
  We conduct simulations to validate both algorithms' performance
  and show their utilities in wireless networking and edge computing.
\end{abstract}


\section{Introduction}\label{sec:intro}

Multi-armed bandits (MAB)~\cite{lai_asymptotically_1985} is a canonical sequential decision
making model for studying the
\textit{exploration-exploitation} trade-off.
In a stochastic MAB, a player pulls one arm among \(K\in\N_+\) arms per time slot
and receives this pulled arm's reward
generated by a random variable.
To maximize the total reward
\update{(i.e., minimize \textit{regret} which is the cumulative reward differences between
    the optimal arm and the chosen arms),}
the player needs to
balance between choosing arms with high uncertainty in
their reward means (exploration) and
choosing the empirical good arms so far (exploitation).
Anantharam \textit{et al.}~\shortcite{anantharam_asymptotically_1987} considered the multi-play MAB (MP-MAB) model in which the player can pull \(M\in\{2,\ldots,K-1\}\) arms from \(K\) arms per time slot.

Recently, a decentralized \emph{multi-player} multi-armed bandits (MMAB) model was proposed~\cite{liu_decentralized_2010,anandkumar_distributed_2011}.
Instead of one player choosing \(M\) arms as in MP-MAB,
there are \(M\) players in MMAB, and each player \textit{independently} pulls one arm from \(K\) arms per time slot.
One particular feature of MMAB is how to define the reward when several players choose the same arm at the same time slot.
One typical setting is that all of them get zero reward, i.e., these players experience a \emph{collision},
and this model was motivated by the cognitive radio network application~\cite{jouini_multi-armed_2009}.
Another setting is that each player gets an independent reward from the arm without influencing each other, i.e., the \emph{non-collision} setting,
which models applications in a distributed system~\cite{landgren_distributed_2016}.

However, to model many real world applications,
we need to consider that arms may have \emph{finite shareable resources} (or capacities), e.g.,
every channel (or arm) in a cognitive radio network can support a finite traffic demand, instead of the restrictive collision assumption.
Similarly, an edge computing node (or arm) in a distributed system can only serve a finite number of tasks, and the over-simplified non-collision assumption cannot be applied.
More concretely, consider a cognitive radio network consisting of \(K\) channels (arms) and \(M\) users (players).
Each user chooses one channel to transmit information per time slot.
If the chosen channel is available,
the user can transmit his information.
Some channels with high (low) bandwidth can support
larger (smaller) number of users per time slot.
To maximize the total transmission throughput, \emph{some players can share these high-bandwidth channels}.
But if too many users choose the same channel (i.e., channel competition occurs), then the channel can only transmit information at its maximum capacity.
Another example is, in mobile edge computing,
each edge computing node (arm) may have multiple yet finite computing units (resources), and thus can also be shared by multiple users (players).
Similar scenarios also appear in many distributed systems, online advertisement placements, and many other applications.


To realistically model many of these applications,
we propose the \textit{\underline{M}ulti-Player \underline{M}ulti-\underline{A}rmed \underline{B}andits with Finite \underline{S}hareable Resources \underline{A}rms} (\texttt{MMAB-SA}) model.
It introduces the \emph{shareable} arms setting into MMAB as an \update{extension} to the collision and non-collision formulations.
This setting not only allows several players to share an arm, but their rewards are dependent and limited by the arm's total shareable resources (quasi-collision).
Formally, each arm is associated with both a ``per-load'' reward random variable \(X_k\) and a finite resources capacity \(m_k\in\N_+\).
When \(a_k\) players choose arm \(k\), they acquire a total reward which is equal to \(\min\{a_k, m_k\}X_k\):
if the number of players pulling the arm does not exceed the arm's shareable resource capacity, i.e., \(a_k \le m_k\), then only \(a_k\) out of these \(m_k\) resources are utilized, e.g., one unit of resources will be given to each player; or otherwise, all \(m_k\) resources would be utilized.
The reward \(X_k\) can model whether a channel is available or not, or the instantaneous computing power of an edge node, etc.
Note that both the ``per-load'' reward mean \(\mu_k\coloneqq \E[X_k]\) and the resources \(m_k\) are unknown to players.

When sharing an arm, players not only observe their reward attained from the arm,
but also some addition feedback about the degree of competition.
In some applications, the arm can interact with players.
For example, in a mobile edge computing system, an arm models a computing node.
When serving players, the computing node can give players feedback about its loading condition.
One type of feedback is the number of players who are sharing the arm in the current time slot, i.e., \(a_k\).
We call this kind of feedback as \emph{``\underline{S}haring \underline{D}emand \underline{I}nformation''} (SDI).
For some other applications, an arm can only provide some limited forms of feedback.
Take a cognitive radio network as an example.
Here, an arm corresponding to a wireless channel cannot send its loading information to users.
Users can only sense whether there exists any sharing of the channel
by other players or not, i.e., \(\1{a_k>1}\).
We name this type of feedback as \emph{``\underline{S}haring \underline{D}emand \underline{A}wareness''} (SDA).
Note that the SDI feedback is more informative than SDA, since knowing \(a_k\) determines \(\1{a_k>1}\).
In SDA, on the other hand, each player only needs to procure \(1\)-bit feedback information,
which is easier to realize in practice.


Recently, several algorithms for MMAB were proposed, e.g., in~\cite{rosenski_multi-player_2016,besson_multi-player_2018,boursier_sic-mmab_2019,wang_optimal_2020}.
However, none of them
can address \texttt{MMAB-SA}.
The reason is that the unknown resource capacity of each arm
and the sharing mechanism complicate the decentralized learning problem:
\textbf{(1) learn resources capacities while sharing arms:}
instead of avoiding collisions for maximizing total reward as in MMAB,
players in our setting \emph{often have to} play the same arm not only to gain higher rewards,
but also to infer the appropriate number of players to pull the arm
since it depends on the arm's resources \(m_k\);
\textbf{(2) information extraction while sharing arms:} instead of taking infrequent collisions only as signals in MMAB (e.g.,~\cite{wang_optimal_2020}), players in our setting needs to \emph{extract information} from
the collision events of arm sharing.
None of the known MMAB algorithms can address these two challenges.
For the first challenge, we propose decentralized algorithms that can explore (learn) arm's maximum resources capacities and exploit (share) good arms while addressing the classic exploration-exploitation trade-off.
For the second challenge, we design two types of communications (and their corresponding algorithms)
to extract information from the degree of resource sharing under the SDI and SDA feedback respectively.
Our contributions are as follows:
\begin{itemize}[leftmargin=10pt,topsep=0pt,itemsep=0.5pt,parsep=0pt,partopsep=0pt]
    \item We propose the \texttt{MMAB-SA} model which significantly expands the application scope of MMAB. It allows several players to share an arm with finite resources.
          We propose the SDI and SDA feedback mechanisms, both of which can be mapped nicely to many real world applications.

    \item We design the \emph{\underline{D}ecentralized \underline{P}arsimonious \underline{E}xploration for \underline{SDI}} (\texttt{DPE-SDI}) algorithm and the \emph{\underline{S}ynchronization \underline{I}nvolved \underline{C}ommunication for \underline{SDA}} (\texttt{SIC-SDA}) algorithm to address \texttt{MMAB-SA} under the SDI and SDA feedback respectively.
          In \texttt{MMAB-SA}, apart from estimating arm's reward means \(\mu_k\), players also need to estimate arms' resources capacity \(m_k\).
          This new estimation task and the shareable arm setting requires more information exchange between players than MMAB.
          To address the issue, we devise two different communication protocols for \texttt{DPE-SDI} and \texttt{SIC-SDA} algorithms.
          We also rigorously prove that \textit{both algorithms have logarithmic regrets} that are tight in term of time horizon.


    \item We conduct simulations to validate and compare the performance of \texttt{DPE-SDI} and \texttt{SIC-SDA}, and apply them in edge computing and wireless networking scenarios respectively.
\end{itemize}


\section{Related Work}\label{sec:related_works}

Our work falls into the research line of multi-player multi-armed bandits (MMAB) first studied by~\cite{liu_decentralized_2010,anandkumar_distributed_2011}.
Recently, Rosenski \textit{et al.}~\shortcite{rosenski_multi-player_2016} proposed the musical chair algorithm which can
distribute players to different arms.
Utilizing collisions, Boursier and Perchet~\shortcite{boursier_sic-mmab_2019} introduced a communication protocol between players.
DPE1~\cite{wang_optimal_2020} was the first algorithm achieving the optimal regret bound.
However,
\update{applying above algorithms to \texttt{MMAB-SA} would result in poor performance (i.e., linear regrets)}
because of the two challenges stated in Section~\ref{sec:intro}'s last paragraph (before contributions).
Besides the collision setting,
the non-collision MMAB model was also studied, e.g., in~\cite{landgren_distributed_2016,martinez-rubio_decentralized_2019}.
\update{However, their
    algorithms relied on the independent reward assumption and
    thus were not applicable in our dependent rewards of shareable arm setting.}

There were several variants of MMAB.
One variant considered the heterogenous reward setting in which players face different reward environments in pulling arms, e.g.,~\cite{kalathil_decentralized_2014,bistritz_distributed_2018,boursier_practical_2020,shi2021heterogeneous}.
Another variant of MMAB models took game theoretical results, e.g., Nash equilibrium, into consideration.
For example, Boursier and Perchet~\shortcite{boursier_selsh_2020} considered selfish players who might
deceive other players and race to occupy the best arms for their own good.
Liu \textit{et al.}~\shortcite{liu_competing_2020} and Sankararaman \textit{et al.}~\shortcite{sankararaman2021dominate} studied the two-sided markets in MMAB,
i.e., both arms and players have preference for choosing each other.
\update{Lastly, Magesh and Veeravalli~\shortcite{magesh2021decentralized} considered a variant that
    when the number of players selecting an arm exceeds a threshold, all player receive zero reward,
    while our model has no such threshold.
    This is
    an important difference because this threshold can be utilized to
    communicate and coordinate.}
{Different from above variants, our model introduces finite shareable resources arms into MMAB.}


\section{Model Formulation}\label{sec:model}

\subsection{The Decision Model}\label{subsec:model_1}

Consider $K\!\in\!\N_+$ arms and $M\!\in\!\N_+\,(M\!<\!K)$ players\footnote{The condition \(M<K\) can be relaxed to \(M<\sum_{k=1}^K m_k\) in \texttt{DPE-SDI} algorithm. We discuss this in Appendix~\ref{appsub:dpe_sdi_init}.}.
The arm \(k\!\in\! [K]\!\coloneqq \!\{1,2,\ldots, K\}\) is associated with $(m_k, X_k)$,
where $m_k \!\in \!\N_+$ \((m_k \!\le\! M)\) denotes its maximal shareable resources capacity, while
$X_k$ is a random variable with support in \([0,1]\) to model the arm's ``per-load'' stochastic reward.
Denote the reward mean as $\mu_k \!\coloneqq\! \mathbb{E}[X_k]$.
Without loss of generality, assume these reward means have a descending order \(\mu_1\! >\! \mu_2 \!>\! \ldots \!>\! \mu_K\).
Reward means $\mu_k$ (including their order) and resource capacities $m_k$ are all unknown to players.

Consider a finite time horizon with length $T\in\N_+$.
In each time slot $t\in \{1,2,\ldots, T\}$,
each player $i \in \{1,2,\ldots, M\}$ selects one arm to pull.
Let \(\bm{a}_t \coloneqq ( a_{1,t}, a_{2,t},\ldots, a_{K,t} ) \) denote the assignment profile of players,
where \(a_{k,t}\) denotes the number of players pulling arm $k$ at time slot $t$.
The assignment profile satisfies $\sum_{k=1}^K a_{k,t} = M$, capturing that
no players are idle in any given time slots.
%
When $a_{k,t}$ players share the arm $k$ at time slots $t$, the total reward they obtain from arm \(k\) is
\[R_{k,t}\coloneqq \min\{a_{k,t}, m_k\} X_{k,t},\]
where the scaler factor $\min\{a_{k,t}, m_k\}$ describes the amount of resources of arm \(k\) utilized by $a_{k,t}$ players.
In other words, when \(a_{k,t} \le m_k\), they utilize \(a_{k,t}\) resources, e.g.,
each player can enjoy one unit of resources;
while if \(a_{k,t} > m_k\), all \(m_k\) resources of the arm \(k\) are shared by \(a_{k,t}\) players.
We focus on maximizing all players' total reward and how the reward \(R_{k,t}\) distributed among \(a_{k,t}\) players is not of interest in this work.

The expected reward of an assignment profile $\bm{a}_t$ is the summation of each arms' reward, which can be expressed as
\[
    \begin{split}
        f(\bm{a}_t) = \E\Big[ \sum\nolimits_{k=1}^K R_{k,t} \Big]
        = \sum\nolimits_{k=1}^K \min\{a_{k,t}, m_k\}\mu_k.
    \end{split}
\]
Hence, \emph{the optimal assignment profile} \(\bm{a}^*\) for maximizing per time slot's reward would be exactly $m_1$ number of players choosing arm $1$, $m_2$ players choosing arm \(2\), and so on, until all players are assigned. This profile can be expressed as
\[
    \bm{a}^*\coloneqq \Big(m_1, m_2, \ldots, m_{L-1}, M - \sum\nolimits_{k=1}^{L-1}m_k, 0, \ldots, 0\Big),
\]
where $L \coloneqq \min\{l: \sum_{k=1}^l m_k \ge M\}$ denotes \emph{the least favored arm index} among the optimal assignment profile.
We call selected arms in the optimal profile \(\bm{a}^*\) as \textit{optimal arms}, and the remaining as \emph{sub-optimal arms}.
Note that $\bm{a}^*$ is unknown to players
because $\mu_k$ and $m_k$ are unknown.

\subsection{Online Learning Problem}\label{subsec:feedback}

When pulling arm $k$, a player not only observes arm \(k\)'s reward, but also some additional feedback
on the degree of competition on this arm.
We consider two types of feedback:
\textbf{Sharing demand information (SDI):} at time slot \(t\), players who pull arm $k$
can also observe the number of players \(a_{k,t}\) that selects the arm $k$;
\textbf{Sharing demand awareness (SDA):} at time slot \(t\), players who pull arm $k$ only know
whether the arm is shared by others or not, i.e., \(\1{a_{k,t} > 1}\).
We note that the SDI feedback is more \update{informative} than SDA because knowing \(a_{k,t}\) directly implies \(\1{a_{k,t} > 1}\).
But SDA is easier to implement
since the player only needs to procure 1-bit information per time.
Note that players cannot freely communicate with each other
and they can only infer the environment from
feedback.

We aim to design decentralized algorithms for players to select arms
and our objective is to maximize the total reward of all players.
In each time slot, the algorithm prescribes an arm for each player given the player's feedback up to time slot $t$, which together forms an assignment profile denoted by \(\bm{a}_t\).
We define regret to quantify the performance of an algorithm when comparing with the optimal profile \(\bm{a}^*\),
\[
    \ERT = \sum\nolimits_{t=1}^T \left(f(\bm{a}^*) - f(\bm{a}_t)\right).
\]
A smaller regret implies that an algorithm achieves a larger reward,
or a reward which is closer to the optimal profile's.


\section{Decentralized Parsimonious Exploration for the SDI Feedback Algorithm}

We begin with the high-level idea of
the \texttt{DPE-SDI} algorithm,
and then present the detailed design of its several phases.
Finally, we show \texttt{DPE-SDI} has a $O(\log(T))$ sub-linear regret.

\subsection{Overview of the Design}

\texttt{DPE-SDI} employs a leader-follower structure:
one player is the leader, and the rest \(M-1\) players are followers.
The leader takes the responsibility of collecting observations (both rewards and SDI feedback) itself
and updates its information (e.g., empirical optimal arms) to followers.
Followers do not need to communicate anything to the leader,
i.e., they only receive information.
\texttt{DPE-SDI} consists of three phases:
\textit{initialization phase},
\textit{exploration-exploitation phase}
and \textit{communication phase}.
The \textit{initialization phase} (line \ref{algo:SDI-init})
detects the number of players
\(M\) and assigns ranks to players.
The player with rank \(1\) becomes the leader.
After the \textit{initialization phase},
\texttt{DPE-SDI} runs the \textit{exploration-exploitation phase} repeatedly
and, when necessary,
runs the \textit{communication phase} in which the leader
updates followers' information.
The \textit{exploration-exploitation phase} (line \ref{algo:SDI-exp})
conducts \emph{exploitation} such that the followers exploit empirical optimal arms,
and three types of explorations:
(1) \emph{parsimonious exploration}, where the leader parsimoniously explores empirical sub-optimal arms;
(2) \emph{individual exploration}, where the leader individually explores empirical optimal arms; and
(3) \emph{united exploration}, where all players pull the same arms
for estimating their maximum resources capacities \(m_k\).
When it ends, the leader updates its estimated parameters
(line \ref{algo:SDI-update} in Algorithm~\ref{alg:dpe_sdi}).
In \textit{communication phase},
the leader
sends his updated parameters to followers (line \ref{algo:SDI-CommSend})
and followers receive them (line \ref{algo:SDI-CommRecv}).
After communication, the algorithm goes back to the \textit{exploration-exploitation phase} (line~\ref{algo:SDI-exp}).
Algorithm \ref{alg:dpe_sdi} outlines the \texttt{DPE-SDI} algorithm.

\noindent\textbf{Notations.}
We use \(i\) to denote each player's rank.
Denote the empirical optimal arm set as \(\mathcal{S}_t\),
the empirical least favored arm index as \(\mathcal{L}_t\),
and the parsimonious exploration arm set as \(\mathcal{E}_t\) (define later).
We also denote the lower and upper confidence bounds of arms' shareable resources as \(\bm{m}_t^l \coloneqq (m_{1,t}^l,\ldots, m_{K, t}^l)\) and \(\bm{m}_t^u \coloneqq (m_{1,t}^u,\ldots, m_{K, t}^u)\).
We summarize the information that leader sends to followers as
\(\bm{\Upsilon}_t \coloneqq  (\mathcal{S}_t, \mathcal{L}_t, \bm{m}^l_t, \bm{m}^u_t) \).
Arms' useful statistics (e.g., empirical means) are aggregated as \(\bm{\Lambda}_t\) (define later).

\begin{algorithm}[htp]
    \caption{\texttt{DPE-SDI}}
    \label{alg:dpe_sdi}
    \begin{algorithmic}[1]
        \State Initialization: \(
        \mathcal{E}_t \gets [K],
        \bm{\Upsilon}_t = (\mathcal{S}_t, \mathcal{L}_t, \bm{m}^l_t, \bm{m}^u_t) \gets
        ([K], 1, (1,\ldots,1)^T, (K,\ldots,K)^T)\), and \(\bm{\Lambda}_t\gets \emptyset\).
        \Statex \LeftComment{Initialization phase}
        \State \((i,M)\gets\texttt{DPE-SDI.Init}()\)
        \label{algo:SDI-init}
        \While{\(t\le T\)}
        \Statex \LeftComment{Exploration-exploitation phase}
        \State
        \(\bm{\Lambda}_t \gets\texttt{DPE-SDI.Explo}(i,M, \mathcal{E}_t, \bm{\Upsilon}_t, \bm{\Lambda}_t)\)
        \label{algo:SDI-exp}
        \State \(\bm{\Upsilon}_{\text{pre}} \gets \bm{\Upsilon}_t\) \Comment{\textit{Record previous info.}}
        \Statex \LeftComment{Communication phase}
        \Statex \LeftComment{Leader performs update and sends info. to followers.}
        \If{\(i=1\)}
        \State \((\mathcal{E}_t, \bm{\Upsilon}_t) \gets\texttt{DPE-SDI.Update}(\bm{\Lambda}_t, \bm{\Upsilon}_t)\)
        \label{algo:SDI-update}
        \If{\(\bm{\Upsilon}_t \neq \bm{\Upsilon}_{\text{pre}}\)}
        \State \texttt{DPE-SDI.CommSend}\((\bm{\Upsilon}_t, \bm{\Upsilon}_{\text{pre}})\)
        \label{algo:SDI-CommSend}
        \EndIf
        \Statex \LeftComment{Followers receive leader's update info.}
        \ElsIf{\(i\neq 1\) and leader informs}
        \State \(\bm{\Upsilon}_t \gets\texttt{DPE-SDI.CommRece}(\bm{\Upsilon}_{\text{pre}})\)
        \label{algo:SDI-CommRecv}
        \EndIf
        \EndWhile
    \end{algorithmic}
\end{algorithm}

\subsection{Initialization Phase}
\label{sec:SDI-iniPhase}

Our initialization phase, consisting of two steps: \textit{rally} and \textit{orthogonalization},
is simpler and more effective than the previous ones in MMAB, e.g.,~\cite{wang_optimal_2020},
due to the advantage of the SDI feedback.
The feedback of the rally step --- all players pull arm \(1\) at time slot \(1\) ---
is equal to the number of players \(M\).
Knowing \(M\), we then orthogonalize \(M\) players to \(M\) arms (i.e., each player individual chooses an arm) and the player's orthogonalized arm index is set as its rank.
We provide the phase's details in Appendix~\ref{appsub:dpe_sdi_init}.

\subsection{Exploration-Exploitation Phase}
\label{sec:SDI-Exploration}

The exploration-exploitation phase
consists of exploitation (followers), individual exploration (leader), parsimonious exploration (leader), and united exploration (all players).
Exploitation and parsimonious/individual exploration are done in parallel,
in which followers exploit the empirical optimal arms while the leader at times deviates to explore empirical sub-optimal ones.
Denote \(\hatbm{a}^*_t \coloneqq (\hat{a}_{1,t}^*, \ldots, \hat{a}_{K,t}^*)\) as the empirical optimal assignment profile at time \(t\).
Given \(\hatbm{a}^*_t\), the empirical optimal arm set \(\mathcal{S}_t\) is \(\{k: \hat{a}_{k,t}^* > 0\}\).
Note that the leader does not need to transmit the profile \(\hatbm{a}_{t}^*\) to followers,
because followers can \textbf{recover \(\hatbm{a}^*_t\) from \(\mathcal{S}_t, \mathcal{L}_t, \bm{m}^l_t\)} (Algorithm~\ref{alg:dpe_sdi_explo}'s line~\ref{alg:dpe_sdi_recover}):
\(\hat{a}_{k,t}^*\) is equal to \(m_{k,t}^l\) when \(k\) is in \(\mathcal{S}_t\) but not equal to \(\mathcal{L}_t\),
and the \(\hat{a}_{\mathcal{L}_t,t}^*\)
is equal to \(M - \sum_{k\in\mathcal{S}_t, k\neq \mathcal{L}_t} m_{k,t}^l\).
Next, we illustrate the phase's four components respectively.

\noindent
\textbf{Exploitation.} Players exploit arms according to the empirical optimal profile \(\hatbm{a}^*_t\), and they play the empirical optimal arms in turn.
To illustrate, we take \((i+t)\Mod{M}\) as their circulating ranks for choosing arms in time slot \(t\).
Note that several players may share the same arm, e.g., when \(\hat{a}^*_{k,t} > 1\) for some \(k\).
To realize this, we apply \textbf{a rotation rule} (Algorithm~\ref{alg:dpe_sdi_explo}'s line~\ref{alg:dpe_sdi_rotate}):
at time slot \(t\), player with rank \(i\) plays arm \(j\) such that \(\sum_{n=1}^{j} \hat{a}_{n,t}^* \ge (i+t)\Mod{M} > \sum_{n=1}^{j-1} \hat{a}_{n,t}^*\). That means players with a circulating rank greater than \(\sum_{n=1}^{j-1} \hat{a}_{n,t}^*\) and no greater than \(\sum_{n=1}^{j} \hat{a}_{n,t}^*\) choose arm \(j\).

\noindent
\textbf{Individual exploration.}
Individual exploration (IE) happens when the explored arm's resource capacity is not exceeded by the number of players pulling the arm, i.e, \(a_{k,t} \le m_k\).
The simplest case of IE is that an arm is played by one player,
but it also includes cases that several players share an arm.
The leader's actions during the exploitation and parsimonious exploration can all be regarded as IEs because \({a}_{k,t} \le m_{k,t}^l\).
Divided by \(a_{k,t}\), IE's rewards can be used to estimate the reward mean \(\hat{\mu}_{k,t} \coloneqq {S}_{k,t}^{\text{IE}}/\tau_{k,t}\), where \({S}_{k,t}^{\text{IE}}\) is the total IE's reward feedback (divided by \(a_{k,t}\)) up to time slot \(t\), and \(\tau_{k,t}\) is the total times of IE for arm \(k\) up to time slot \(t\).

\noindent
\textbf{Parsimonious exploration.}
Denote $u_{k,t}$ as the KL-UCB index of arm \(k\) at time slot \(t\)~\cite{cappe_kullback-leibler_2013}:
\(
u_{k,t} = \sup\{q\ge 0: \tau_{k,t}\kl(\hat{\mu}_{k,t}, q)\le f(t)\},
\)
where $f(t) = \log (t) + 4 \log\log (t)$.
The parsimonious exploration arm set \(\mathcal{E}_t\) consists of empirical sub-optimal arms whose KL-UCB indexes are larger than the least favored arm \(\mathcal{L}_t\)'s empirical mean, i.e., \(\mathcal{E}_t\coloneqq \{k: u_{k,t} \ge \hat{\mu}_{\mathcal{L}_t,t}, \hat{a}^*_{k,t} = 0\}.\)
When the leader is supposed to play the least favored arm \(\mathcal{L}_t\) for exploitation (also IE), with a probability of \(0.5\),
he will uniformly choose an arm from \(\mathcal{E}_t\) to explore; or otherwise, the leader exploits the arm \(\mathcal{L}_t\).
This parsimonious exploration idea was first made in~\cite{combes_learning_2015} for learning-to-rank algorithms and was further utilized by~\cite{wang_optimal_2020} in MMAB.

\noindent
\textbf{United exploration.}
United exploration (UE) refers to all \(M\) players rallying on an arm so as to acquire a sample of \(m_kX_k\), i.e., the reward when the arm's resources are fully utilized.
Collecting this kind of observation is crucial in estimating an arm's shareable resources capacity \(m_k\)  (see Appendix~\ref{appsub:dpe_sdi_leader_update}).
Let \(\mathcal{S}'_t \coloneqq \{k\in \mathcal{S}_t:m_{k,t}^{l} \neq m_{k,t}^{u}\}\) denote the set of empirical optimal arms whose resource capacities have not been learnt.
In one round of exploration-exploitation phase, all \(M\) players unitedly explore each arm in \(\mathcal{S}'_t\) once.
Denote \(S_{k,t}^{\text{UE}}\) as the total reward of arm \(k\) in UE up to time \(t\), \(\iota_{k,t}\) as the total times of UE for arm \(k\) up to time \(t\), and thus the empirical mean estimate of \(\E[m_kX_k]\) is \(\hat{\nu}_{k,t}\coloneqq S_{k,t}^{\text{UE}} / \iota_{k,t}\).
The output statistics \(\bm{\Lambda}_t\) consists of
each arm's IE's total reward
\(\bm{S}^{\text{IE}}_t \coloneqq (S^{\text{IE}}_{1,t},\dots,S^{\text{IE}}_{K,t})\),
IE times \(\bm{\tau}_t\coloneqq(\tau_{1,t},\dots, \tau_{K,t})\),
and their UE counterparts: \(\bm{S}^{\text{UE}}_t \coloneqq (S^{\text{UE}}_{1,t},\dots,S^{\text{UE}}_{K,t}), \bm{\iota}_t\coloneqq(\iota_{1,t},\dots, \iota_{K,t})\).

\begin{algorithm}[htp]
    \caption{\texttt{DPE-SDI.Explo}\((i,M, \mathcal{E}_t, \bm{\Upsilon}_t, \bm{\Lambda}_t)\)}
    \label{alg:dpe_sdi_explo}
    \begin{algorithmic}[1]
        \State \algorithmicensure \(\bm{\Lambda}_t \coloneqq (\bm{S}^{\text{IE}}_t, \bm{S}^{\text{UE}}_t, \bm{\tau}_t, \bm{\iota}_t)\)
        \State \(\hatbm{a}_t^* \gets \text{Recover}(\mathcal{S}_t, \mathcal{L}_t, \bm{m}_t^l)\) \label{alg:dpe_sdi_recover}
        \For{\(M\) times}
        \State \(j_t\gets \text{Rotate}(i,M,\hatbm{a}_t^*,t)\) \label{alg:dpe_sdi_rotate}
        \Statex \LeftComment{Parsimonious exploration by leader}
        \If{\(i=1\)}
        \If{\(\mathcal{E}_t \neq \emptyset\) and \(j_t = \mathcal{L}_t\)}
        \State{w.p. $1/2$,  Play arm \(l\sim\mathcal{E}_t\) uniformly}
        \State{w.p. $1/2$,  Play arm \(j_t\)}
        \Statex \LeftComment{Individual exploration by leader}
        \EndIf
        \ElseLine{Play arm \(j_t\)}
        \State \(S^{\text{IE}}_{k,t} \gets S^{\text{IE}}_{k,t}  + {r_{k,t}}/{\hat{a}^*_{k,t}}, \tau_{k,t}\gets \tau_{k,t} + 1\)
        \Statex \LeftComment{Exploitation by followers}
        \EndIf
        \ElseIfThen{\(i\neq 1\)}{Play arm \(j_t\)}

        \State \(t\gets t+1\).
        \EndFor
        \Statex \LeftComment{United exploration by all players}
        \State \(\mathcal{S}'_t\gets \{k\in\mathcal{S}_t: m_{k,t}^{l} \neq m_{k,t}^{u}\}.\)
        \For{\(k\in\mathcal{S}'_t\)}
        \State All players pull arm \(k.\)
        \IfThen{\(i=1\)}{\(S^{\text{UE}}_{k,t} \gets S^{\text{UE}}_{k,t}  + r_{k,t}\); \(\iota_{k,t} \gets \iota_{k,t} + 1\)}
        \State \(t\gets t+1\)
        \EndFor
    \end{algorithmic}
\end{algorithm}

Algorithm~\ref{alg:dpe_sdi_explo} outlines the exploration-exploitation phase.
After the phase, the leader updates \(\bm{\Upsilon}_t\) via \texttt{DPE-SDI.Update}.
Its detailed procedure is deferred to Appendix~\ref{appsub:dpe_sdi_leader_update}.

\subsection{Communication Phase}\label{subsec:dpe_sdi_comm}

After updating \(\bm{\Upsilon}_t=(\mathcal{S}_t, \mathcal{L}_t,\bm{m}^{l}_t, \bm{m}^{u}_t)\),
if there are parameter changes,
leader will initialize a round of communication to update followers' \(\bm{\Upsilon}_t\).
For simplicity, we illustrate the communication as six steps.
Some potential improvements of the phase
are discussed in Appendix~\ref{appsub:dpe_sdi_comm}.
\begin{itemize}[leftmargin=15pt,topsep=0pt,itemsep=0.5pt,parsep=0pt,partopsep=0pt]
    \item Initial step: leader signals followers to communicate.
    \item 2nd step: notify followers arms to be removed from \(\mathcal{S}_t\).
    \item 3rd step: notify followers arms to be added in \(\mathcal{S}_t\).
    \item 4th step: notify followers the least favored arm \(\mathcal{L}_t\).
    \item 5th step: update followers' resources' lower bounds \(\bm{m}^l_t\).
    \item 6th step: update followers' resources' upper bounds \(\bm{m}^u_t\).
\end{itemize}
\noindent
\textbf{Initial step.}
When the leader needs to update followers' information,
he stops at the arm \(\hat{k}^*_t\) with the highest empirical mean in subsequent \(M\) time slots.
Meanwhile, followers take turns to pull arms of \(\mathcal{S}_t\) (exploitation).
Because the leader doesn't follow the rotation rule and keeps pulling arm \(\hat{k}^*_t\).
When followers playing arm \(\hat{k}^*_t\),
they would observe \(a_{\hat{k}^*_t,t}\! >\! \hat{a}^*_{\hat{k}^*_t,t}\),
which violates the rule,
and realize a communication starts.

\noindent
\textbf{2nd-6th steps.}
Each of the remaining 5 steps consists of exactly \(K\) time slots.
In each step,
followers synchronize to play the same arms and rotate over all \(K\) arms once.
When these followers select an arm \(k\), the number of players sharing arm \(k\) (SDI feedback) is \(M-1\).
When the leader wants to update information related to arm \(k\),
he will choose arm \(k\) as followers play the arm.
Then all followers take the SDI feedback becoming \(M\) as a signal for updating arm \(k\)'s related information.
The detailed updates of these steps (including conditions and pseudo-codes) are presented in Appendix~\ref{appsub:dpe_sdi_comm}.

\subsection{Regret Bound of \texttt{DPE-SDI}}

In the following theorem,
we state a regret upper bound for \texttt{DPE-SDI} algorithm.
Its proof is presented in Appendix~\ref{app:dpe_sdi_regret}.
\vspace{-3pt}
\begin{theorem}
\label{thm:dpe_sdi_regret}
For any given set of parameters \(K, M, \bm{\mu}, \bm{m}\) and \(0<\delta<\min_{1\le k \le K-1}(\mu_k - \mu_{k+1})/2\), the regret of \emph{\texttt{DPE-SDI}} is upper bounded as follows:
\begingroup
\allowdisplaybreaks
\begin{align}
    \ERT \le & \!\sum\nolimits_{k=L+1}^K\! \frac{(\mu_L -\mu_k)(\log T + 4 \log(\log T))}{\kl(\mu_k+\delta,\mu_L-\delta)} \nonumber \\
             & + \sum\nolimits_{k=1}^M {49w_km_k^2}{\mu_k^{-2}}\log T    \label{eq:regret:SDI}                                      \\
             & +\! 588KMm_M^2 \mu_M^{-2}\!\log T \!+ \!156M^3K^3(4\!+\!\delta^2) \nonumber
\end{align}
\endgroup
where \(w_k\coloneqq f(\bm{a}^*) + \mu_1 - (m_k+1)\mu_k\) is the highest cost of one round of IE and UE.
\end{theorem}
\vspace{-3pt}
Eq.(\ref{eq:regret:SDI})'s first two terms
correspond to parsimonious exploration and united exploration of the exploration-exploitation phase respectively.
The last two terms cover the cost of initialization phase, communication phase, exploiting sub-optimal arms, and falsely estimating resources.
\update{Note that \texttt{DPE-SDI}'s \(O(\log T)\) regret is tight in term of \(T\), as it matches
    MMAB's regret lower bound \(\Omega(\log T)\)~\cite{anantharam_asymptotically_1987}.}

\vspace{-6pt}
\section{Synchronization Involved Communication for the SDA Feedback Algorithm} \label{sec:sic_sda}
\vspace{-5pt}





Recall that in \texttt{DPE-SDI.CommSend}'s initial step,
the leader signals followers to start communication
via abnormally deviating the empirical optimal assignment \(\hatbm{a}_t^*\).
That is, leader sticks to an arm
such that the number of players (leader pluses followers) pulling the arm \(a_{k,t}\)
is greater than the expected.
While under the SDI feedback followers can sense the increase of \(a_{k,t}\),
in the SDA's \(\1{a_{k,t} \!>\! 1}= 1\) feedback followers may fail.
Not being able to sense the difference makes the communication protocol devised for \texttt{DPE-SDI}
not applicable to the SDA feedback.
To address this issue, we propose the \texttt{SIC-SDA} algorithm with a phase-based communication protocol.
Because it is phase-based, communications start and end at a pre-defined scheme and
do not need initial steps.

The \texttt{SIC-SDA} algorithm consists of four phases: \emph{initialization phase,} \emph{exploration phase,}
\emph{communication phase,} and
\emph{exploitation phase.}
The objective of \emph{initialization phase} is to probe the total number of players \(M\) and assign ranks to players.
After the \emph{initialization phase}, \texttt{SIC-SDA} runs the \emph{exploration phase} and \emph{communication phase}
iteratively (i.e., in a loop).
In the loop, whenever an optimal arm is identified with high confidence,
\texttt{SIC-SDA} will assign players to exploit the arm's resources.
In other words, these players leave the loop and execute \emph{exploitation phase} so to zoom in on the optimal arm.
As the algorithm runs, players in the two-phase loop will gradually enter the \emph{exploitation phase}.
Due the limited space, we defer the \texttt{SIC-SDA} algorithm's detail to Appendix~\ref{app:sic_sda}.

Lastly, we state that the \texttt{SIC-SDA} algorithm enjoys a \(O(\log T)\) regret in Theorem~\ref{thm:sic_sda_regret}. Its proof is in Appendix~\ref{app:sic_sda_regret}.
\vspace{-10pt}
\begin{theorem}
    \label{thm:sic_sda_regret}
    For any given set of parameters \(K, M, \bm{\mu}, \bm{m}\), the regret of \emph{\texttt{SIC-SDA}} is upper bounded as follows:
    \begin{align}
        \ERT\! \le       & c_1\!\sum\nolimits_{k=L+1}^K \frac{M \log T}{\mu_L-\mu_k} \! +  \!c_2\!\sum\nolimits_{k=1}^M \!\frac{m_k^2}{\mu_k^2}\log T \nonumber          \\
        + c_3 K^2 \log^2 & \!\left(\!\max \!\left\{\!\frac{\log (T)}{(\mu_L - \mu_{L+1})^2}, \frac{m_M^2\log(T)}{\mu_M^2} \!\right\}\!\right), \label{eq:sic_sda_regret}
    \end{align}
    where \(c_1, c_2, c_3\) are universal positive constants.
\end{theorem}
\vspace{-5pt}




\section{Simulations \& Applications}

\subsection{Synthetic Data Simulations}

We consider an environment with \(9\) arms and \(6\) players, which is the common simulation setting in MMAB literatures~\cite{boursier_sic-mmab_2019,wang_optimal_2020}.
The ``per-load'' reward random variables follow the Bernoulli distribution.
Their reward means are a random permutation of a decreasing array.
The array starts from \(0.9\) and each successor is smaller than its predecessor by \(\Delta\),
e.g., when \(\Delta=0.025\), it is \([0.90, 0.875, 0.85, 0.825, 0.80, 0.775, 0.75. 0.725, 0.70]\).
These \(9\) arms' resource capacities are \([3, 2, 4, 2, 1, 5, 2, 1, 3]\).
For each experiment, we calculate average and standard variance (as error bar) over \(50\) simulations.
Except \texttt{DPE-SDI} and \texttt{SIC-SDA}, we also apply the \texttt{SIC-SDA} algorithm to SDI feedback setting and call this as \texttt{SIC-SDI}.

In Figure~\ref{fig:regret}, we compare these three algorithms' performance for different \(\Delta\) values.
First, one can see that the \texttt{DPE-SDI} algorithm
has a much smaller regret than the other two.
The intuition is that \texttt{DPE-SDI} is based on the KL-UCB algorithm which is theoretically better than the elimination algorithms utilized by \texttt{SIC-SDA} and \texttt{SIC-SDI}~\cite{cappe_kullback-leibler_2013}.
Second, we compare \texttt{DPE-SDI} and \texttt{SIC-SDI} under the same SDI feedback.
When \(\Delta\) becomes smaller, the gap between \texttt{DPE-SDI}'s regret and \texttt{SIC-SDI}'s also becomes smaller.
This observation implies that when optimal arms are difficult to identify (e.g., \(\Delta=0.001\)),
the efficacy of both algorithms is similar.
Lastly, we note that the \texttt{SIC-SDI} algorithm outperforms \texttt{SIC-SDA},
which is consistent with the fact that the SDI feedback is more informative than SDA.

\begin{figure}
    \centering
    \subfloat[$\Delta=0.001$]{\includegraphics[width=0.25\textwidth]{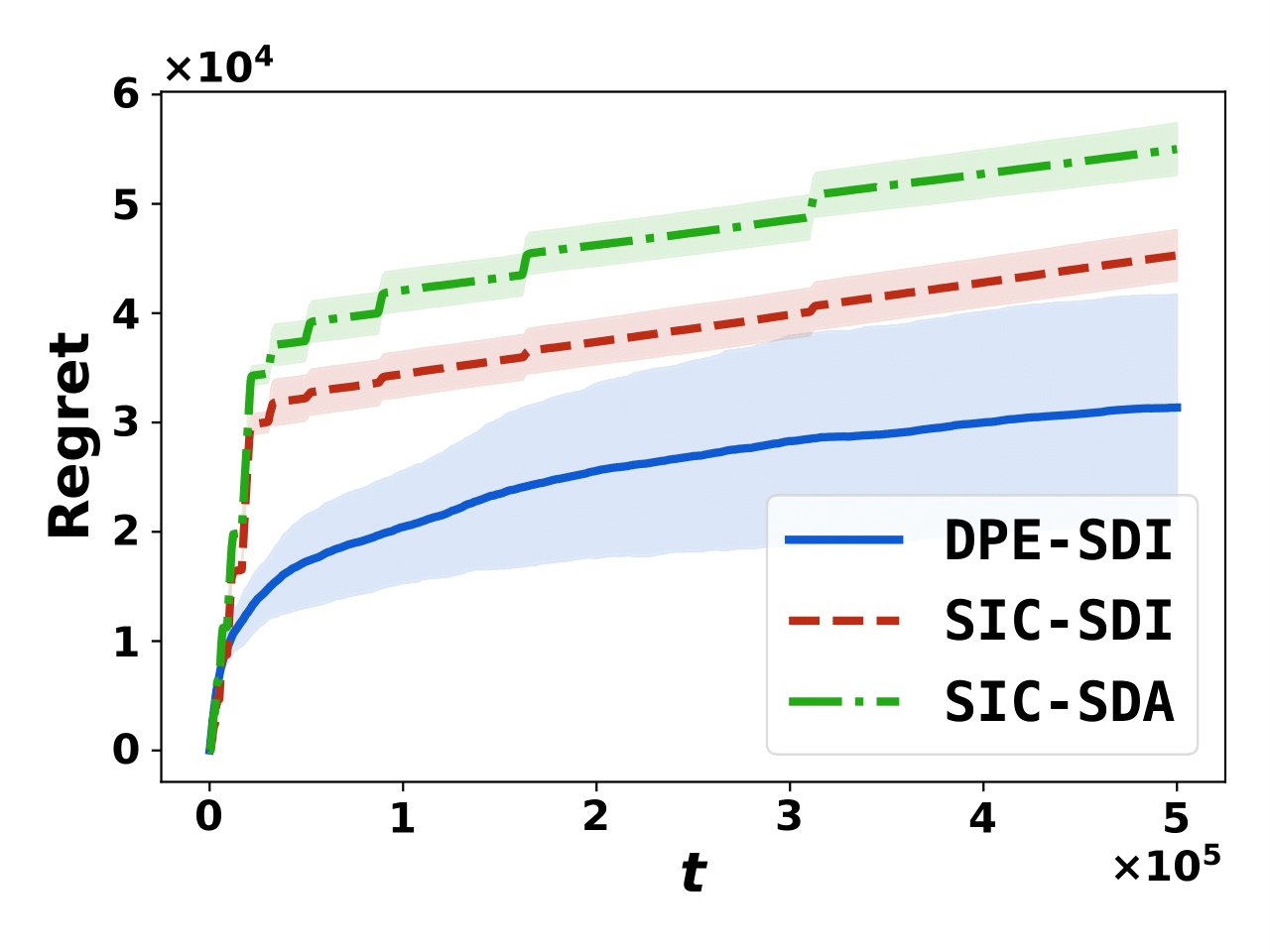}}
    \subfloat[$\Delta=0.012$]{\includegraphics[width=0.25\textwidth]{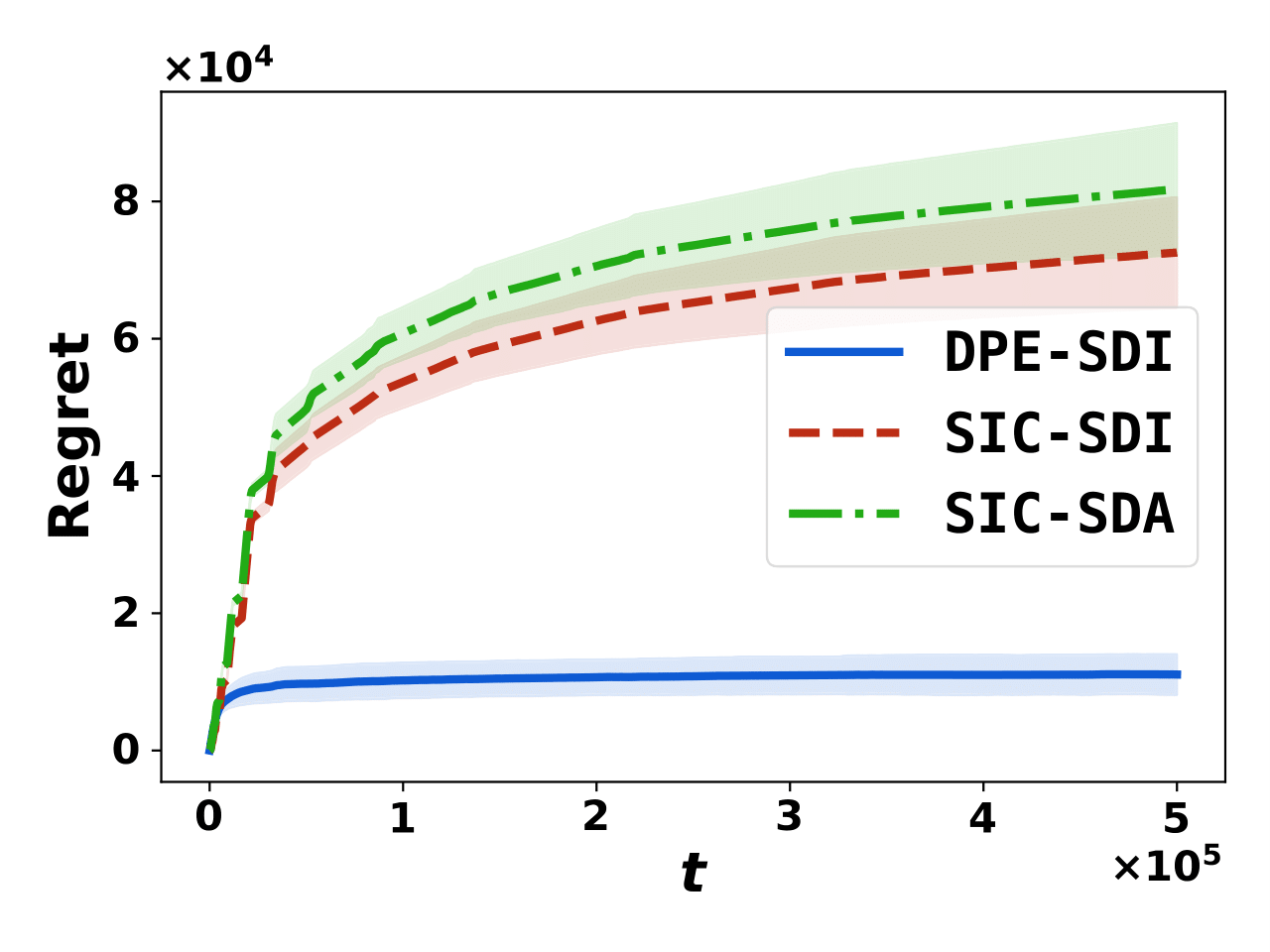}}
    \\[-2ex]
    \subfloat[$\Delta=0.025$]{\includegraphics[width=0.25\textwidth]{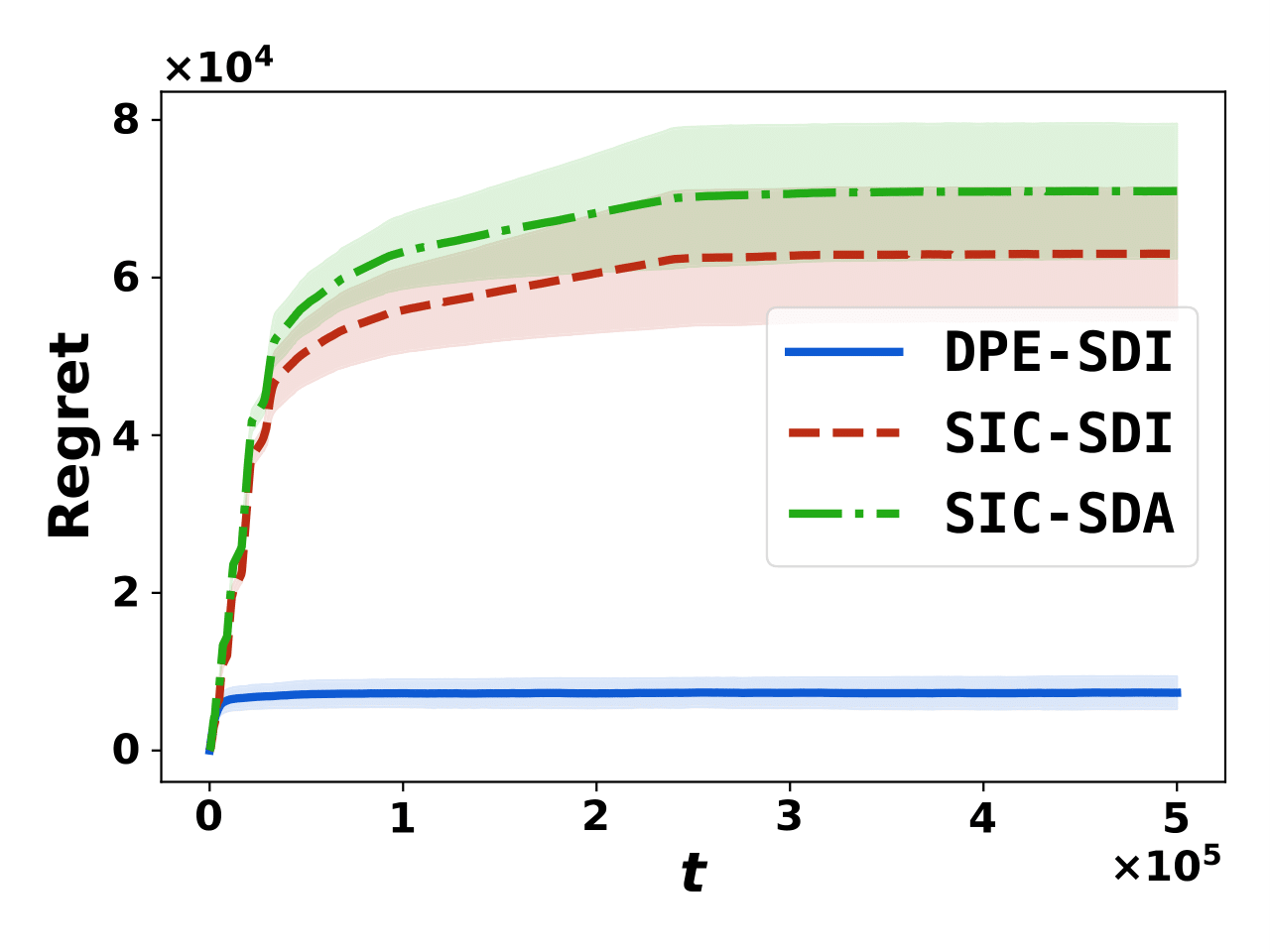}}
    \subfloat[$\Delta=0.037$]{\includegraphics[width=0.25\textwidth]{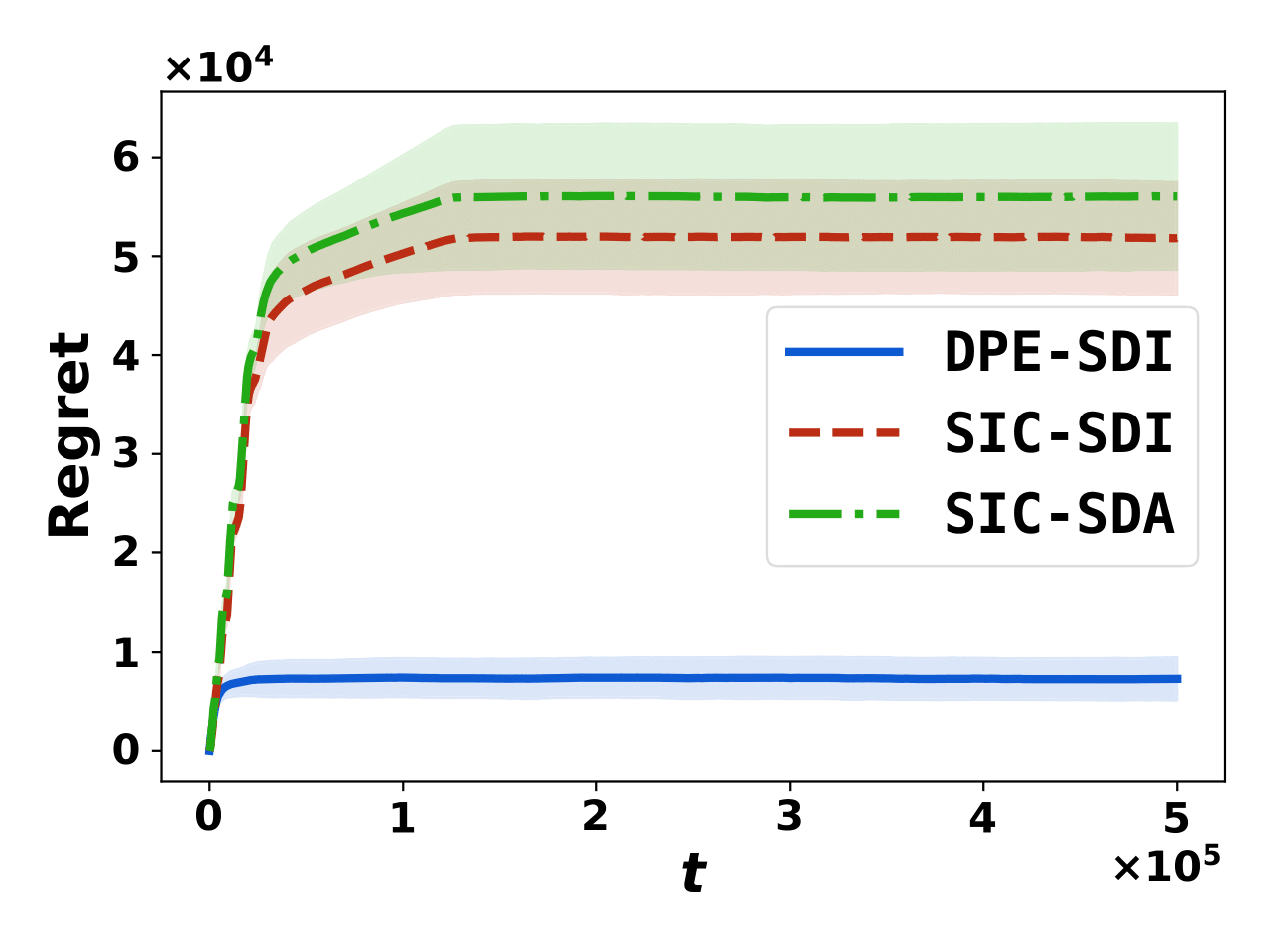}}
    \caption{Synthetic data simulations}
    \label{fig:regret}
\end{figure}

\subsection{Real World Applications}

We consider two real world scenarios:
(1) an edge computing system for the SDA feedback and (2) a 5G/4G wireless networking system for the SDI feedback.

\noindent
\textbf{Edge computing. }
The edge computing system Precog~\cite{drolia_precog_2017} was designed for image processing.
Its edge nodes
contain mobile devices and personal desktop computers, which fits the SDA setting.
According to their specifications, we consider \(7\) edge computing nodes (arms):
\centerline{\tabcolsep=0.1cm
    \begin{tabular}{|c||ccccccc|}
        \hline
        CPU Speed (GHz) & 1.5 & 2.1 & 1.2 & 2.5 & 2.0 & 1.3 & 2.6 \\
        \hline
        \# of CPU Cores & 3   & 2   & 4   & 2   & 1   & 2   & 3   \\ \hline
    \end{tabular}
}
To scale the CPU speed as ``pre-load'' reward mean in \([0,1]\), we divide them by \(3\).
The number of CPU cores corresponds to arms' shareable resource capacities \(m_k\).
We assume there are \(6\) parallel tasks (players) in the system.

\noindent
\textbf{5G \& 4G network. }
5G started to serve consumers from 2019 and will coexist with 4G for a long time.
When a smartphone uses wireless service, it needs to choose between 5G and 4G networks.
The smartphone only has the SDI feedback as it doesn't know the exact number of users
using a particular base station.
Narayanan \textit{et al.}~\shortcite{narayanan_first_2020} measured the 5G performance on smartphone and compared it with 4G.
We pick the \(8\) parallel TCP connections as our typical setting, in which
the 5G's throughput (THR) is around \(8\) times higher than 4G's,
and 4G's round-trip time (RTT) latency is around \(4\) times slower than 5G's.
From~\cite{narayanan_first_2020}'s experimental results, we consider an environment consisting of two 5G base stations (underlined) and eighteen 4G base stations (\(20\) arms) and \(18\) smartphones (\(18\) players):
\centerline{\tabcolsep=0.075cm
    \begin{tabular}{|c||cccccccccc|}
        \hline
        RTT (100ms)   & \underline{{1.2}} & \underline{{1.1}} & 4.2 & 4.0 & 4.5 & 3.5 & 5.0 & 4.2 & 5.5 & 3.9 \\ \hline
        THR (100Mbps) & \underline{{9.2}} & \underline{{8.1}} & 1.2 & 1.2 & 1.4 & 1.1 & 1.3 & 1.2 & 1.1 & 1.4 \\
        \hline
        RTT (100ms)   & 4.8               & 5.5               & 3.7 & 4.7 & 3.2 & 5.1 & 4.4 & 5.3 & 4.9 & 4.1 \\ \hline
        THR (100Mbps) & 1.0               & 1.1               & 1.2 & 1.0 & 1.3 & 1.2 & 1.0 & 1.1 & 1.3 & 1.2 \\
        \hline
    \end{tabular}
}
We use their RTT latencies' reciprocals as arms' ``per-load'' reward means and their THR's integer rounding as
arms' maximal resource capacities.

We apply the \texttt{DPE-SDI} and \texttt{SIC-SDA} algorithms to solve these two scenarios respectively.
We also implement two heuristic policies according to each player's own observations:
the \texttt{Highest-Reward} policy
in which players choose the arm with highest empirical reward mean,
and
the \texttt{Idlest-Arm} policy in which players select the arm that is most infrequently shared with others.
Figure~\ref{fig:real_world}a and \ref{fig:real_world}b show that in the edge computing scenario,
\texttt{DPE-SDI} outperforms other two heuristic policies.
In Figure~\ref{fig:real_world}c and \ref{fig:real_world}d for the 5G/4G wireless networking scenario,
when time slot is small, the \texttt{Highest-Reward} policy is better than \texttt{SIC-SDA}.
Because the policy can detect two powerful 5G base stations immediately and utilize them from the very beginning.
But in the long run, the \texttt{SIC-SDA} algorithm will find the optimal profile and finally can outperform both heuristic policies.

\begin{figure}
    \centering
    \subfloat[Edge computing (SDI): Regret]{\includegraphics[width=0.25\textwidth]{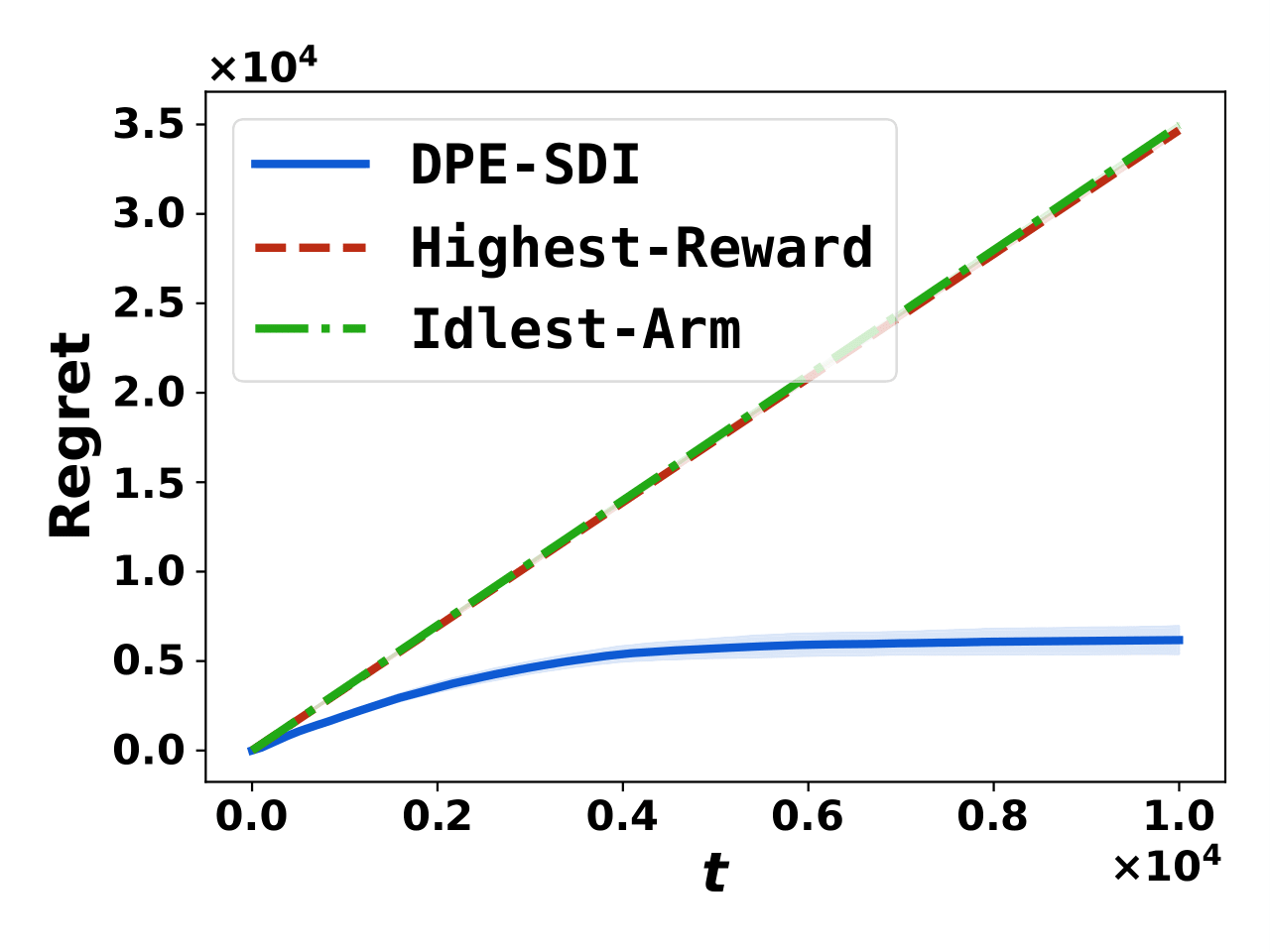}}
    \subfloat[Cumulative CPU usage]{\includegraphics[width=0.25\textwidth]{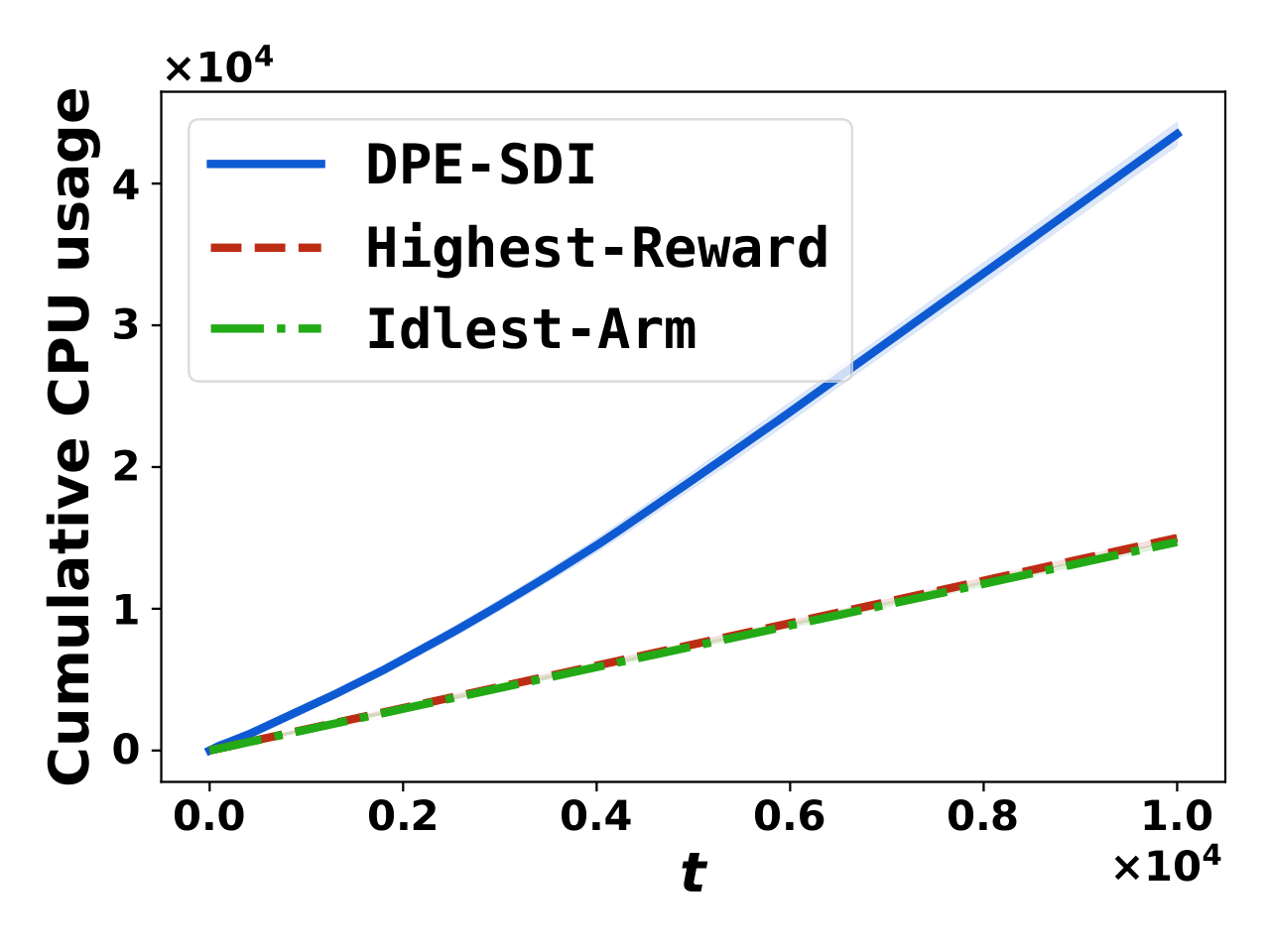}}\\[-2ex]
    \subfloat[5G/4G network (SDA): Regret]{\includegraphics[width=0.25\textwidth]{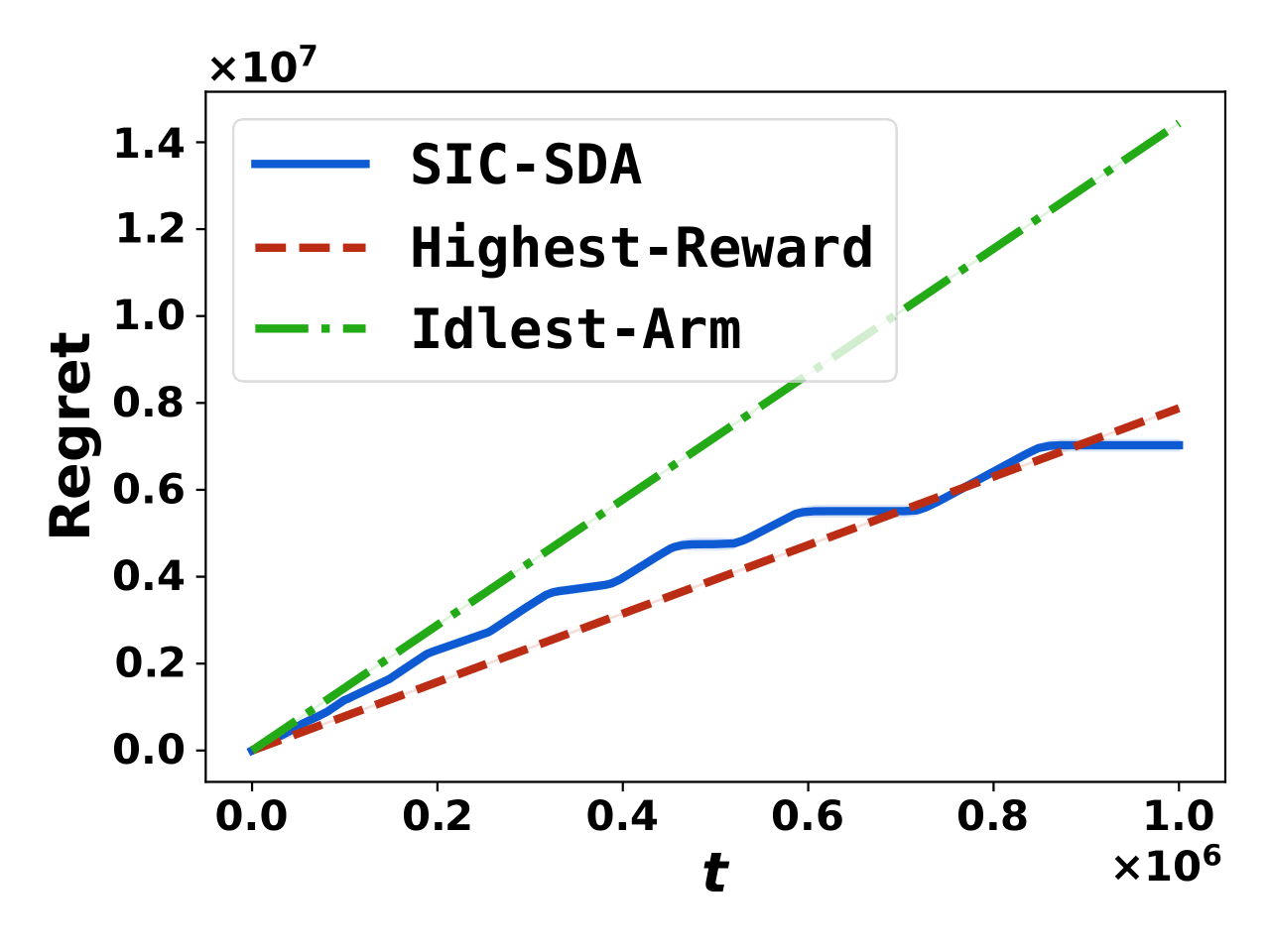}}
    \subfloat[Total throughput]{\includegraphics[width=0.25\textwidth]{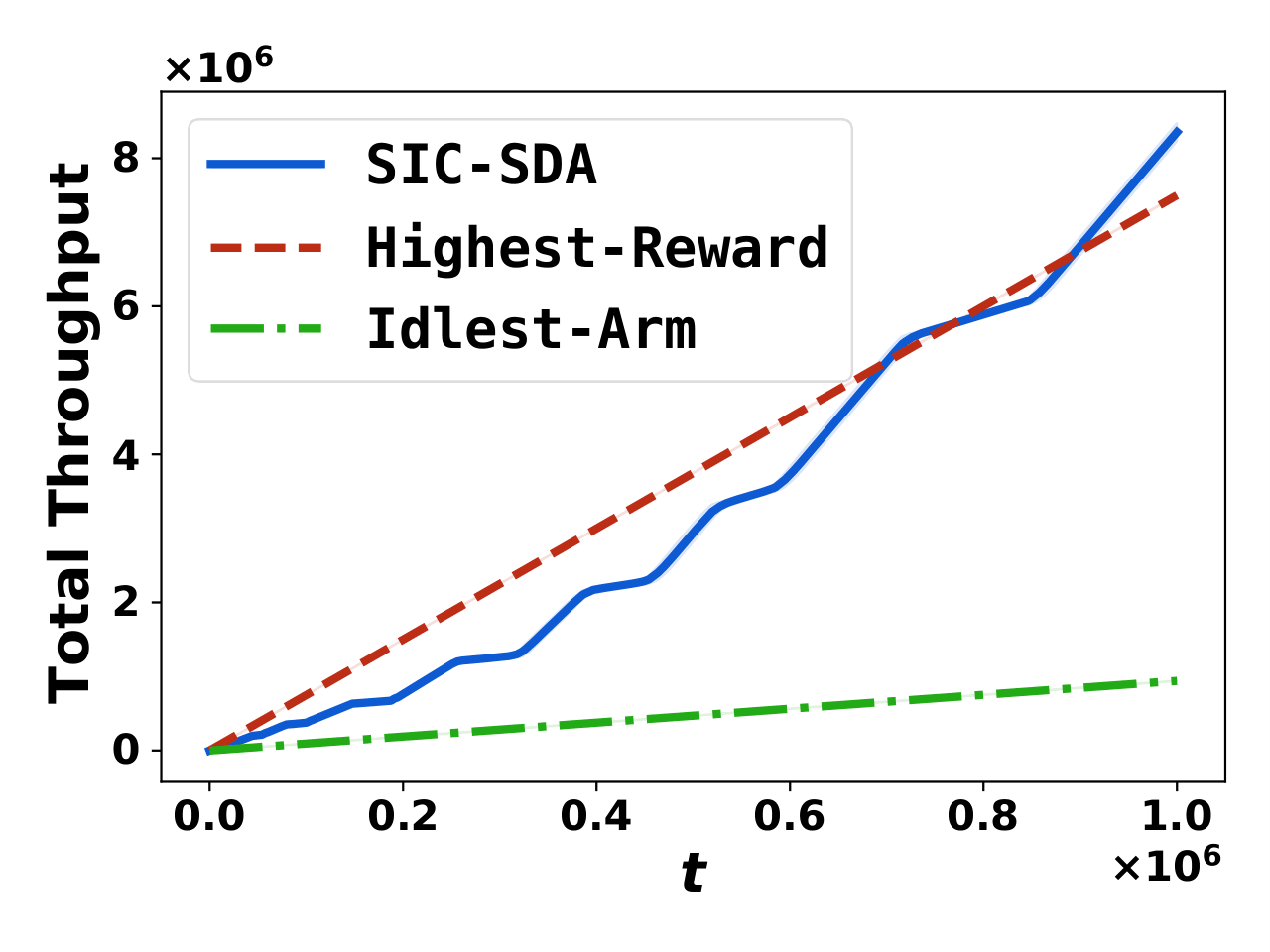}}
    \caption{Real world data simulations}
    \label{fig:real_world}
\end{figure}




\section{Conclusion}

We propose a novel multi-player multi-armed bandits model in which several players can share an arm, and their total reward is the arm's ``per-load'' reward multiplied by the minimum between the number of players selecting the arm and the arm's maximum shareable resources.
We consider two types of feedback: sharing demand information and sharing demand awareness.
Under both types of feedback respectively,
we design the \texttt{DPE-SDI} and \texttt{SIC-SDA} algorithms coordinating decentralized players
to explore arms' reward means and resource capacities,
exploit the empirical optimal allocation profile,
and communicate with each other.
Especially, \texttt{DPE-SDI} has a simpler initialization phase than the previous algorithms' in MMAB due to the SDI feedback.
We prove that both algorithms achieve the logarithmic regrets that are tight in term of time horizon.
We also compare these two algorithms' performance
using synthetic and real world data
and show their utilities in wireless networking and edge computing.

\section*{Acknowledgement}

The work of John C.S. Lui was supported in part by SRFS2122-4S02.
The work of Hong Xie was supported by Chongqing Talents: Exceptional Young Talents Project
(cstc2021ycjh-bgzxm0195).
Hong Xie is the corresponding author.

\bibliographystyle{named.bst}
\bibliography{bibliography}

\clearpage
\appendix

\section{Additional Related Works: Comparison with Stochastic Non-Linear Fractional Equality Knapsack Problem} \label{app:related_works}


The stochastic non-linear fractional equality knapsack (NFEK), e.g.,~\cite{yazidi2018solving}
studied a resources sharing problem similar to this work.
However, our \texttt{MMAB-SA} model is different from theirs twofold:
(1) \texttt{MMAB-SA} is in the decentralized setting where players cannot easily communicate with each other, while NFEK is in the centralized setting;
(2) Our algorithms' logarithmic regret bounds not only guarantee convergences to optimal allocations as \(\lim_{T\to\infty}(\log T) / T = 0\), but also that their total cost (regret) is tight in term of \(T\) comparing to the logarithmic regret lower bound in~\cite{anantharam_asymptotically_1987},
while algorithms for stochastic NFEK only guaranteed the convergence to the optimal allocation.




\section{Algorithmic Detail of \texttt{DPE-SDI}}\label{app:dpe_sdi}

\subsection{Initialization Phase}\label{appsub:dpe_sdi_init}

The initialization phase is presented in Algorithm~\ref{alg:dpe_sdi_init}.
After sensing \(M\) from Rally in line~\ref{algline:rally}, it orthogonalizes and assigns ranks to all players.
Each player starts with a rank \(M+1\).
The algorithm runs in rounds and each round contains \(M+1\) time slots.
In the beginning of a round, if the player's rank is \(M+1\), then he uniformly chooses an arm \(i\in \{1,2,\dots,M\}\) to play,
and if his rank \(i\) is less than \(M+1\), he plays arm \(i\).
If player with rank \(M+1\) receives the feedback \(a_{j,t} = 1\), he updates his rank to \(j\);
otherwise, there are other players sharing the arm and his rank is still \(M+1\).
In the round's subsequent \(M\) time slots, each player always chooses his rank's corresponding arm
except that at the \((i+1)\)th time slot he chooses arm \(M+1\).
If each player was orthogonally assigned with a rank in \(\{1,2,\ldots,M\}\), in these subsequent \(M\) time slots, no sharing (i.e., \(a_{k,t} > 1\)) would happen, which will end the orthogonalization sub-phase.
Or otherwise, all players will experience arm sharing at least once and start a new round of orthogonalization.
This orthogonalization approach is similar to \cite{wang_optimal_2020}'s.
We show the expected time slots of this orthogonalization is \(\E[T_0]\le M^4+2M^3+M^2\) in Appendix~\ref{appsub:dpe_sdi_regret_init}.

\begin{algorithm}[htp]
    \caption{\(\texttt{DPE-SDI.Init}()\)}
    \label{alg:dpe_sdi_init}
    \begin{algorithmic}[1]
        \State\algorithmicensure rank \(i\), the number of players \(M\)
        \State \(M \gets\text{Rally}()\) \Comment{determine the number of players}\label{algline:rally}
        \State \(i\gets \text{Orthogonal}(M)\) \Comment{assign rank to players} \label{algline:rankassignment}
    \end{algorithmic}
\end{algorithm}

\update{\noindent\textbf{Relax condition \(M<K\) to \(M<\sum_{k\in[K]}m_k\). }
    With some modifications, the \texttt{DPE-SDI} algorithm can work when the number of players \(M\) is greater than
    the number of arms \(K\) (but not greater than \(\sum_{k\in[K]}m_k\),
    or otherwise, it would become a trivial case).
    The key modification is in the initial phase:
    while probing the total number of players via rallying is still valid,
    the original orthogonalization for assigning ranks needs to be modified.
    Note that one stage of orthogonalization can assign ranks to at most \(K-1\) players.
    So, when \(M>K\), one needs to implement \(\ceil{\frac{M}{K-1}}\) stages
    of orthogonalizations to assign each player a rank.
    We take the first stage of orthogonalization to illustrate the process.
    In this stage, each player starts with a rank of \(K\),
    uniformly choose an arm \(i\in \{1,2,\dots,K-1\}\),
    and follow the same process as the original orthogonalization's first step.
    In this round's subsequent \(K-1\) time slots, players with rank \(i\) smaller than \(K\) pull the arm whose
    index is equal to its rank except that at the \((i+1)\)th time slot, they pull arm \(K\).
    If each player experience either no arm sharing or arm sharing with exactly \(M-K+2\) players in
    these subsequent \(K-1\) time slots,
    then it means the first stage of orthogonalization finishes:
    \(K-1\) out of \(M\) players are assigned with ranks \(\{1,2,\dots, K-1\}\)
    while the remaining players are still with rank \(K\).
    Or otherwise, they start a new round consisting of \(K\) time slots.
    This new condition for terminating one stage of orthogonalization is
    a key modification from the original orthogonalization.
    In the second stage of orthogonalization, players with rank \(i<K\) keep pulling arm \(K\) while other players with rank \(K\)
    follow the same procedure as the first stage.
    At the end of the second stage, players newly assigned with a rank \(i<K\) will be assigned
    with a rank \(i + (K-1)\) so as to differentiate them from ranks assigned in previous stages.
    After \(\ceil{\frac{M}{K-1}}\) stages of orthogonalizations, each player is assigned with a
    rank from \(1\) to \(M\).
    Except this modification in \texttt{initial phase}, the other phases of \texttt{DPE-SDI} works for the \(M < \sum_{k\in [K]} m_k\) case smoothly.
}

\subsection{Leader's Parameter Update}\label{appsub:dpe_sdi_leader_update}

At the end of the exploration and exploitation phase, the leader will update its estimates \(\bm{\Upsilon}_t\) as that outlined in Algorithm~\ref{alg:dpe_sdi_update_stats}.
\(\bm{\Upsilon}_t\) consists of the empirical optimal arm set \(\mathcal{S}_t\), the empirical least favor arm \(\mathcal{L}_t\), and arms' shareable resources' lower and upper confidence bounds \(\bm{m}_t^l\) and \(\bm{m}_t^u\).

The leader updates the shareable resources \(m_k\)'s lower and upper confidence bounds via the observations from IEs and UEs as follows:
\begin{align}
    m_{k,t}^l = &
    \max\!\left\{\! m_{k,t}^l,\! \ceil{\frac{\hat{\nu}_{k,t}}{\hat{\mu}_{k,t} + \phi(\tau_{k,t}, \delta) + \phi(\iota_{k,t}, \delta)}} \! \right\},\label{eq:lower_bound_update} \\
    m_{k,t}^u = &
    \min\!\left\{\! m_{k,t}^u, \!\floor{\frac{\hat{\nu}_{k,t}}{\hat{\mu}_{k,t} - \phi(\tau_{k,t}, \delta) - \phi(\iota_{k,t}, \delta)}}\! \right\},\label{eq:upper_bound_update}
\end{align}
where the function
\(
\phi(x,\delta) \!\coloneqq\! \sqrt{\left(1+\frac{1}{x}\right)\frac{\log(2\sqrt{x+1}/\delta)}{2x}}.
\)
Their derivations are deferred to Appendix~\ref{app:update_ci}.

The leader updates the empirical optimal arm set \(\mathcal{S}_t\) and the empirical least favor arm \(\mathcal{L}_t\) via the empirical optimal profile \(\hatbm{a}_t^*\) (line~\ref{alg:dpe_sdi_update_stats:optimal_arms},~\ref{alg:dpe_sdi_update_stats:least_favored_arms} of Algorithm~\ref{alg:dpe_sdi_update_stats}).
To find the optimal profile \(\bm{a}^*\), we define function \(\Oracle: (\bm{\mu}, \bm{m})\to \bm{a}^*\) according to the rule stated in Section~\ref{subsec:model_1}.

\begin{algorithm}[htp]
    \caption{\texttt{DPE-SDI.Update}\((\bm{\Lambda}_t{=}(\bm{S}^{\text{IE}}_t, \bm{S}^{\text{UE}}_t, \bm{\tau}_t, \bm{\iota}_t), \bm{\Upsilon}_t)\)}
    \label{alg:dpe_sdi_update_stats}
    \begin{algorithmic}[1]
        \State \algorithmicensure \((\mathcal{E}_t, \bm{\Upsilon}_t = (\mathcal{S}_t, \mathcal{L}_t, \bm{m}^l_t, \bm{m}^u_t))\)
        \If{\(i=1\)} \Comment{\textit{Leader}}
        \State \(\hat{\mu}_{k,t}\gets {S}_{k,t}^{\text{IE}}/ {\tau}_{k,t}, \hat{\nu}_t\gets {S}_{k,t}^{\text{UE}} / {\iota}_{k,t},\,\forall k\)
        \State Updates \(\bm{m}^{l}_t, \bm{m}^{u}_t\) via Eq.(\ref{eq:lower_bound_update})(\ref{eq:upper_bound_update})
        \State{\(\hatbm{a}^*_t\gets \Oracle(\hat{\bm{\mu}}_t, \bm{m}^l_t)\)}\Comment{update optimal profile}\label{alg:dpe_sdi_update_stats:oracle_update}
        \State{\(\mathcal{S}_t\gets \{k: \hat{a}_{k,t}^* > 0\}\)\label{alg:dpe_sdi_update_stats:optimal_arms}}\Comment{update optimal arms}
        \State\multiline{\(\mathcal{L}_t\gets \argmin_k\{\hat{\mu}_{k}:a_{k,t} > 0\}\) \label{alg:dpe_sdi_update_stats:least_favored_arms}}
        \State{$\mathcal{E}_t\gets \{k: u_{k,t} \ge \hat{\mu}_{L_t,t}, a_{k,t} = 0\}$}
        \EndIf
    \end{algorithmic}
\end{algorithm}



\subsection{Communication Phase}\label{appsub:dpe_sdi_comm}
The \texttt{DPE-SDI.CommSend} and \texttt{DEP-SDI.CommRece} are presented in Algorithm~\ref{alg:dpe_sdi_commsend} and Algorithm~\ref{alg:dpe_sdi_commrece}. Table~\ref{tab:dep_sdi_comm_detail} provides Algorithm~\ref{alg:dpe_sdi_commsend} \texttt{DPE-SDI.CommSend}'s 2nd-6th steps' condition (line~\ref{alg:dpe_sdi_commsend:line:condition}) and Algorithm~\ref{alg:dpe_sdi_commrece} \texttt{DPE-SDI.CommRece}'s 2nd-6th steps' update (line~\ref{alg:dpe_sdi_commrece:line:update}).
\begin{table}[htp]
    \centering
    \caption{\texttt{DPE-SDI.CommSend} and \texttt{CommRece}'s details}
    \label{tab:dep_sdi_comm_detail}\tabcolsep=0.05cm
    \begin{tabular}{|c||c|c|}
        \hline
        Step & \(j\)-th step condition (leader)                        & \(j\)-th step update (follower)                      \\ \hline
        2nd  & \(k\in\mathcal{S}_{\text{pre}}\setminus\mathcal{S}_t\)  & \(\mathcal{S}_t\gets \mathcal{S}_t \setminus \{k\}\) \\
        3rd  & \(k\in\mathcal{S}_t\setminus \mathcal{S}_{\text{pre}}\) & \(\mathcal{S}_t\gets \mathcal{S}_t \cup \{k\}\)      \\
        4th  & \(k = \mathcal{L}(t)\)                                  & \(\mathcal{L}_t \gets k\)                            \\
        5th  & \(m_{k,t}^l > m_{k,\text{pre}}^l\)                      & \(m_{k,t}^l \gets m_{k,t}^l+1\)                      \\
        6th  & \(m_{k,t}^u < m_{k,\text{pre}}^u\)                      & \(m_{k,t}^u \gets m_{k,t}^u - 1\)                    \\
        \hline
    \end{tabular}
\end{table}

\begin{algorithm}[htp]
    \caption{\(\texttt{DPE-SDI.CommSend}(\bm{\Upsilon}_t, \bm{\Upsilon}_{\text{pre}})\)}
    \label{alg:dpe_sdi_commsend}
    \begin{algorithmic}[1]
        \For{\(M\) times}\Comment{initial step}
        \State{Play arm \(\hat{k}_t^*\) with highest empirical mean}
        \EndFor
        \For{\(j\in\{2,3,4,5,6\}\)}\Comment{2nd-6th step}
        \For{\(k\in [K]\)}
        \If{\(j\)-th step's condition}\Comment{send update signal}\label{alg:dpe_sdi_commsend:line:condition}
        \State Play arm \(k\)
        \Else\Comment{no update signal}
        \State Play arm \((k+1)\Mod{M}\)
        \EndIf
        \EndFor
        \EndFor
    \end{algorithmic}
\end{algorithm}
\begin{algorithm}[htp]
    \caption{\texttt{DPE-SDI.CommRece}\((\bm{\Upsilon}_{\text{pre}})\)}
    \label{alg:dpe_sdi_commrece}
    \begin{algorithmic}[1]
        \If{\(a_{\hat{k}_t^*,t} > \hat{a}^*_{\hat{k}_t^*,t}\)}
        \State Start communication  \Comment{initial step}
        \State Wait till 2nd step
        \For{\(j\in\{2,3,4,5,6\}\)}\Comment{2nd-6th step}
        \For{\(k\in [K]\)}
        \State Play arm \(k\)\Comment{detect update signals}
        \If{\(a_{k,t} =M\)}
        \State \(j\)-th step's update\label{alg:dpe_sdi_commrece:line:update}
        \EndIf
        \EndFor
        \EndFor
        \EndIf
    \end{algorithmic}
\end{algorithm}

The six-step communication of \texttt{DPE-SDI} needs \(M+5K\) time slots.
These time slots can be reduced to \(4M+K\) via two improvements:
\begin{itemize}
    \item The 2nd and 3nd step can be merged as one. Because an arm is either removed or added to the empirical optimal arm set \(\mathcal{S}_t\). So, if a notified arm is already in \(\mathcal{S}_t\), then it should be removed from \(\mathcal{S}_t\); or otherwise, it should be added into \(\mathcal{S}_t\).
    \item The 4th, 5th and 6th steps can be reduced to \(\abs{\mathcal{S}_t} < M\) steps.
          After update \(\mathcal{S}_t\) in the 2nd and the 3rd steps,
          one can choose the least favored arm \(\mathcal{L}_t\) among \(\mathcal{S}_t\) and only update arms in \(\mathcal{S}_t\)'s shareable resource capacities' confidence bounds.
          So, in the 4th, 5th and 6th steps, players only need to rotate over arms in \(\mathcal{S}_t\).
\end{itemize}

\section{Algorithmic Detail of \texttt{SIC-SDA}} \label{app:sic_sda}


\noindent \textbf{Notation.} 
Denote \exparm~as each player's exploitation arm: when \(\exparm=-1\), it means that the player is \textit{active} --- in the loop of \textit{exploration phase} and \textit{communication phase}; 
or otherwise, 
the player is exploiting the \exparm~arm in \textit{exploitation phase}.
At the start, all players' \exparm s are initiated as \(-1\) (i.e., active).
We use \(p\) to denote the phase number while is initiated as \(1\). 
After each loop of \textit{exploration phase} and \textit{communication phase}, the phase number \(p\) is increased by one. 

The \texttt{SIC-SDA} algorithm's abstract framework is presented in Algorithm~\ref{alg:sic_sda}.
Its detailed sub-procedures are illustrated in the following sub-sections. 
Before going into their details, we take the place to further illustrate the \textit{communication phase} and 
define its several key terms and notations.
Recall that arms' statistics are denoted as \(\bm{\Lambda}_t\), and its composition is the same of \texttt{DPE-SDI}'s. 
We use \(\bm{\Lambda}_t(j)\) to indicate the rank-\(j\) player's statistics when it is necessary to distinct 
different players' statistics. 
The information maintained in each player at the \texttt{SIC-SDA} algorithm is separated as \(\bm{\Phi}_t\) and \(\bm{\Psi}_t\),
where \(\bm{\Phi}_t\) is updated by the leader who collects all followers' observations,
and, after leader disseminating the \(\bm{\Phi}_t\) information to each followers,
each player can then update their own \(\bm{\Psi}_t\) independently.
The detail compositions of \(\bm{\Phi}_t\) and \(\bm{\Psi}_t\) will be defined later. 

The \textit{communication phase} consists of two communication sub-procedures: 
\texttt{SIC-SDA.CommBack} (line~\ref{alg:sic_sda:commback_leader},\ref{alg:sic_sda:commback_follower} in Algorithm~\ref{alg:sic_sda}), where followers send back their former phases' statistics \(\bm{\Lambda}_t\) to leader, 
and \texttt{SIC-SDA.CommForth} (line~\ref{alg:sic_sda:commforth_leader},\ref{alg:sic_sda:commforth_follower} in Algorithm~\ref{alg:sic_sda}), where the leader sends its latest parameters \(\bm{\Phi}_t\) to followers;
and two parameters update sub-procedures:
\texttt{SIC-SDA.AccRej} (line~\ref{alg:sic_sda:line:accrej} in Algorithm~\ref{alg:sic_sda}), where the leader --- after acquiring all followers' statistics \(\bm{\Lambda}_t\) --- updates its own parameters \(\bm{\Phi}_t\),
and \texttt{SIC-SDA.Update} (line~\ref{alg:sic_sda:line:update} in Algorithm~\ref{alg:sic_sda}), where all players --- after followers receiving the leader's latest \(\bm{\Phi}_t\) information --- update their own information \(\bm{\Psi}_t\) and \exparm~arms.

\begin{algorithm}[htp]
    \caption{\texttt{SIC-SDA}}
    \label{alg:sic_sda}
    \begin{algorithmic}[1]
        \State Initialization: \(\exparm \gets -1, p\gets 1, \bm{\Psi}_t, \bm{\Phi}_t, \bm{\Lambda}_t\)
        \Statex \LeftComment{Initialization phase}
        \State \((i,M)\gets\texttt{SIC-SDA.Init()}\)
        \While{\(\exparm = -1\)}
            \Statex \LeftComment{Exploration phase}
            \State \(\bm{\Lambda}_t{(i)} \gets \texttt{SIC-SDA.Explore}(i, p, \bm{\Psi}_t)\)
            \Statex \LeftComment{Communication phase}
            \Statex \LeftComment{Leader receives follower's statistics , updates its info., and sends its info. to followers.}
            \If{\(i=1\)}
                \State \(\bm{\Lambda}_t (1)\gets\texttt{SIC-SDA.CommBack}()\)\label{alg:sic_sda:commback_leader}
                \State \(\bm{\Phi}_t
                \gets\texttt{SIC-SDA.AccRej}(\bm{\Lambda}_t{(1)}, \bm{\Psi}_t)\)\label{alg:sic_sda:line:accrej}
                \State \(\texttt{SIC-SDA.CommForth}(\bm{\Phi}_t)\) \label{alg:sic_sda:commforth_leader}
            \Statex \LeftComment{Followers send its statistics and receive leader's update.}
            \ElsIf{\(i\neq 1\)}
               \State \(\texttt{SIC-SDA.CommBack}(i, p, \bm{\Lambda}_t{(i)})\) \label{alg:sic_sda:commback_follower}
               \State \(\bm{\Phi}_t\gets\texttt{SIC-SDA.CommForth}()\) \label{alg:sic_sda:commforth_follower}
            \EndIf
            \State \((\!\exparm,\! \bm{\Psi}_t\!){\gets}\texttt{SIC-SDA\!.\!Update}(\!i,\! \bm{\Phi}_t, \!\bm{\Psi}_t\!)\) \label{alg:sic_sda:line:update}
            \State \(p \gets p+1\)
        \EndWhile
        \Statex \LeftComment{Exploitation phase}
        \State Play the \(\exparm\) to the end.
    \end{algorithmic}
\end{algorithm} 

In the parameters update of \texttt{SIC-SDA.AccRej}, some arms may be rejected as sub-optimal arms, 
while some may be accepted as optimal arms.
Except that these rejected arms will not be pulled anymore and these accepted arms will be exploited,
the remaining arms --- called \emph{active arms} --- will be further explored.
Denote  \(\mathcal{A}_t\) as a subset of arms to be accepted, \(\mathcal{R}_t\) as a subset of arms to be rejected,
\(\mathcal{L}_t\) as the least favored arm (if any), \(\mathcal{K}_t\) as the set of active arms at time slot \(t\),
and \(K_t\coloneqq \abs{\mathcal{K}_t}\) as the number of active arms. 
Apart from players exploiting accepted arms, the remaining players,
called \emph{active players},
continue to explore active arms.
We denote \(M_t\) as the number of active players at time slot \(t\).
We summarize these information that updated by the leader itself in \texttt{SIC-SDA.AccRej} as \(\bm{\Phi}_t \coloneqq (\mathcal{A}_t, \mathcal{R}_t, \mathcal{L}_t, \bm{m}_t^l, \bm{m}_t^u)\),
where \(\bm{m}_t^l,\bm{m}_t^u\) represent the lower and upper confidence intervals for arms' maximum resources capacities.
Denote the other information that is updated by each player itself 
as \(\bm{\Psi}_t \coloneqq (\bm{b}_t, \mathcal{K}_t, M_t)\),
where \(\bm{b}_t\in R^{K}\) is an auxillary vector for IE in the exploration phase,
\(\mathcal{K}_t\) is the active arm set, and \(M_t\) is the number of active players.

\subsection{Initialization Phase}\label{appsub:sic_sda_init}

\texttt{SIC-SDA}'s initialization phase (Algorithm~\ref{alg:sic_sda_init}) 
is to estimate the number of players \(M\) and to assign a rank to each player,
which is similar to~\cite{wang_optimal_2020}'s.
It 
consists of the orthogonalization and the rank assignment. 
The idea of orthogonalization (line~\ref{alg:sic_sda_init:ortho} in Algorithm~\ref{alg:sic_sda_init}) is illustrated in \texttt{DPE-SDI}'s initialization phase (Appendix~\ref{appsub:dpe_sdi_init}), except that \texttt{SIC-SDA} needs to orthogonalize all players to \(K-1\) arms 
due to not knowing the number of players \(M\). 
After the orthogonalization steps, each player has a different rank \(k\in \{1,2,\dots, K-1\}\). 
The rank assignment (line~\ref{alg:sic_sda_init:rank}) aims to 
transform these \(M\) players' ranks from shattering in \(\{1,2,\dots, K-1\}\)
to neatly ordered ranks in \(\{1,2,\dots,M\}\).
It takes \(2K-2\) consecutive time slots.
We number them as \(\{t_1, t_2, \ldots, t_{2K-2}\}\). 
During the rank assignment, players chooses arms according to the rule:
for player with rank \(k\in\{1,2,\dots,K-1\}\),
when \(t_s\in \{t_1,\ldots, t_{2k}\}\cup \{t_{K+k}, \ldots, t_{2K-2}\}\), he plays arm \(k\);
otherwise, when \(t_s \in \{t_{2k+1},\ldots, t_{K+k-1}\}\), he plays arm \(s-k\).
Following this rule, any pair of two players will share an arm exactly once. 
One can observe that the total number of sharing demand awareness in the first \(2k\) time slots plus one, i.e., \(\sum_{t\in\{t_1,\ldots, t_{2k}\}}\1{a_{k,t}>1}+1\), would become the arm's final rank \(i\in\{1,2,\dots,M\}\). 
The total number of sharing awareness in all \(2K-2\) time slots plus one, i.e., \(\sum_{t\in\{t_1,\ldots, t_{2K-2}\}}\1{a_{k,t}>1}+1\), is equal to the number of players \(M\). 

\begin{algorithm}[htp]
    \caption{\(\texttt{SIC-SDA.Init}()\)}
    \label{alg:sic_sda_init}
    \begin{algorithmic}[1]
        \State \algorithmicensure rank \(i\), the number of players \(M\)
        \State \(k \gets \text{Orthogonal}(K-1)\) \label{alg:sic_sda_init:ortho}
        \State \((i,M)\gets \text{RankAssign}(k)\) \label{alg:sic_sda_init:rank}
    \end{algorithmic}
\end{algorithm}

\subsection{Exploration Phase}\label{appsub:sid_sda_explore}


The exploration phase consists of (individual exploration) IE and (united exploration) UE. 
While the UE part is similar to \texttt{DPE-SDI}'s, the IE part is different.
In IE, when the number of active players \(M_t\) is greater than the number of active arms \(K_t\), some players need to share arms, i.e., \(a_{k,t}>1\). 
Recall that updating \(\bm{S}_t^{\text{IE}}\) needs to know \(a_{k,t} (<m_k)\), 
the number of players sharing arms, 
which \(a_{k,t}\) is unobservable under the SDA feedback.
To infer \(a_{k,t}\) in IE, we design the following \textbf{arm selection rule}.
If a player's rank \(i\) is no greater then \(K_t\), he plays the \(((i+t)\Mod{K_t})\)-th active arm; otherwise, he chooses the \(j\)-th active arm such that \(\sum_{n=1}^j (m_{\sigma(n)}^l - 1) \ge i > \sum_{n=1}^{j-1} (m_{\sigma(n)}^l - 1)\), where \(\sigma(n)\) denotes the \(n\)-th active arm among \(\mathcal{K}_t\). 
The inequality maps the rank \(i\) higher than \(K_t\) to the \(j\)-th active arm in \(\mathcal{K}_t\) such that 
the \(a_{k,t}\) --- the number of players pulling arm \(k\) --- does not exceed the arm's shareable resources capacity.

Algorithm~\ref{alg:sic_sda_explore} illustrates the \texttt{SIC-SDA} algorithm's exploration phase. 
Its IE is in line~\ref{alg:sic_sda_explore:individual_explore_start}-\ref{alg:sic_sda_explore:individual_explore_end} and UE is in line~\ref{alg:sic_sda_explore:united_explore_start}-\ref{alg:sic_sda_explore:united_explore_end}.
Vector \(\bm{b}_t\in\N^K\) records how each active arm is shared in IE according to the arm selection rule above.
We note that IE and UE are phase-based: each arm is explored \(2^p\) times where \(p\) is the phase's index. 

\begin{algorithm}[htb]
    \caption{\(\texttt{SIC-SDA.Explore}(i, p, \bm{\Psi}_t = (\bm{b}_t,\mathcal{K}_t, M_t))\)}
    \label{alg:sic_sda_explore}
    \begin{algorithmic}[1]
        \State \algorithmicensure \(\bm{\Lambda}_t = (\bm{S}^{\text{IE}}_t, \bm{S}^{\text{UE}}_t, \bm{\tau}_t, \bm{\iota}_t)\)
        \Statex \LeftComment{Individual exploration}
        \State \(K_t \gets \abs{\mathcal{K}_t}\)
        \For{\(K_t2^p\) times} \label{alg:sic_sda_explore:individual_explore_start}
            \If{\(i \le K_t\)}
                \State \(j\gets (i+t)\Mod{K_t}\)
                \State Play the \(j\)-th active arm.
                \State \(S^{\text{IE}}_{j,t}\gets S^{\text{IE}}_{j,t} + {r_{j,t}}/{b_{k,t}}\); \(\tau_{j,t}\gets \tau_{j,t}+1\) \label{alg:line:update_ie}
            \ElsIf{\(i > K_t\)}
                \State\multiline{Play the \(j\)-th active arm such that \(\sum_{n=1}^j (m_{\sigma(n),t}^l - 1) \ge i > \sum_{n=1}^{j-1} (m_{\sigma(n),t}^l - 1)\) where \(\sigma(n)\) denotes the \(n\)-th active arm.}
            \EndIf
        \EndFor\label{alg:sic_sda_explore:individual_explore_end}
         \Statex \LeftComment{United exploration}
         \State \(\mathcal{K}'_t \gets \{k\in\mathcal{K}_t: m_{k,t}^l \neq m_{k,t}^u\}\)\label{alg:sic_sda_explore:united_explore_start}
        \For{\(2^p\) times} 
            \ForAll{ arm \(k \in \mathcal{K}'_t\)} 
                \State Play arm \(k\)
                \State \(S^{\text{UE}}_{k,t}\gets S^{\text{UE}}_{k,t} + r_{k,t}\); \(\iota_{k,t}\gets \iota_{k,t}+1\)
            \EndFor
        \EndFor\label{alg:sic_sda_explore:united_explore_end}
    \end{algorithmic}
\end{algorithm}

\subsection{Communication Phase}\label{appsub:sic_sda_comm}

\texttt{SIC-SDA}'s communication phase consists of two parts: (1) followers send their IE accumulative reward \(\bm{S}^{\text{IE}}_t\) back to the leader, (2) leader sends his updated information \(\mathcal{A}_t\), \(\mathcal{R}_t\), \(\mathcal{L}_t, \bm{m}^l_t, \bm{m}^u_t\) to followers. 

\noindent
\textbf{\texttt{SIC-SDA.CommBack}:} \emph{Followers send \(\bm{S}^{\text{IE}}_t\) to the leader.}
This part utilizes the \(\1{a_{k,t} \ge 1}\) feedback to transmit bit information.
In particular, leader always play arm \(1\) as the communication arm.
The follower sends bit \(1\) to leader by playing arm \(1\), which generates an arm sharing indication, and sends bit \(0\) by avoiding arm \(1\).

\begin{algorithm}
    \caption{\(\texttt{SIC-SDA.CommBack}(i, p, \bm{S}^{\text{IE}}_t \in \bm{\Lambda}_t)\)}
    \label{alg:sic_sda_commback}
    \begin{algorithmic}[1]
        \For{\((j,k)\in \{2,3, \ldots, \min\{M_t, K_t\}\}\times [K_t]\)}
        \Statex \LeftComment{Leader's part}
            \If{\(i=1\)}
                \State \multiline{Leader receives \(S^{\text{IE}}_{k,t}(j)\) from follower with rank \(j\) in \(p+1\) time slots.} \label{alg:sic_sda_comback:line:leader_receive}
                \State \(S_{k,t}^{\text{IE}}(1) \gets S_{k,t}^{\text{IE}}(1) + S^{\text{IE}}_{k,t}(j)\)
                \State \(\tau_{k,t} \gets \tau_{k,t} + 2^p\)
        \Statex \LeftComment{Follower's part}
            \ElsIf{\(i=j\)} 
                \State \multiline{The follower with rank \(j\) sends \(S^{\text{IE}}_{k,t}(j)\) to leader in \(p+1\) time slots.} \label{alg:sic_sda_comback:line:follower_send}
            \EndIf 
        \EndFor
    \end{algorithmic}
\end{algorithm}

The \(\texttt{SIC-SDA.CommBack}(i, p, \bm{S}^{\text{IE}}_t)\) is presented in Algorithm~\ref{alg:sic_sda_commback}. 
In the \(n\)-th phase, each arm is individually explored \(2^n\) times. 
So, a follower's accumulated reward of any arm till the end of phase \(p\) is at most \(2^{p+1} - 1\), and thus can be transmitted from the follower to leader in \(p+1\) time slots via the \(\1{a_{k,t} \ge 1}\) SDA indicator, i.e., one time slot transmits one bit. 
The communication has a mesh-grid structure: for time slots corresponding to \((j,k)\in [2, \min\{M_t, K_t\}]\times [K_t]\), follower with rank \(i = j\) will send arm \(k\)'s IE reward summation \(S^{\text{IE}}_k\) to leader. 
We use \(S^{\text{IE}}_{k,t}(j)\) to indicate the accumulative IE rewards of the player with rank \(j\) when 
more than one players' \(S^{\text{IE}}_{k,t}\) appear simultaneously. 
In line~\ref{alg:sic_sda_comback:line:leader_receive}'s \(p+1\) time slots, 
the leader always chooses arm \(1\) as its communication arm. 
The received \(S^{\text{IE}}_{k,t}(j)\)'s binary expression's \(n\)-th position is \(1\) if the leader receives an arm sharing signal at the \(n\)-th time slot among the \(p+1\) slots, i.e., \(\1{a_{k,t}>1}=1\); 
and it is \(0\) otherwise. 
In line~\ref{alg:sic_sda_comback:line:follower_send}'s \(p+1\) time slots, 
the follower chooses arm \(1\) at the \(n\)-th time slot among the \(p+1\) slots if his \(S^{\text{IE}}_{k,t}(j)\)'s binary expression's \(n\)-th position is equal to \(1\), and chooses arm \(2\) otherwise.

\noindent
\textbf{\texttt{SIC-SDA.CommForth}:} \emph{Leader sends \(\mathcal{A}_t,\mathcal{R}_t,\mathcal{L}_t,\bm{m}^l_t,\) (i.e., \(\bm{\Phi}_t\)) \( \bm{m}^u_t\) to followers. }
This part is similar to \texttt{DPE-SDI}'s six-step communication. 
As \texttt{SIC-SDA}'s communication is phase-based, it does not need the initial block. 
Each of this part's \(5\) steps updates one of \(\mathcal{A}_t\), \(\mathcal{R}_t\), \(\mathcal{L}_t, \bm{m}^l_t, \bm{m}^u_t\) in \(K_t\) time slots.

The \texttt{SIC-SDA.CommForth}\((\mathcal{A}_t, \mathcal{R}_t, \bm{m}_t^l, \bm{m}_t^u)\) is illustrated in Algorithm~\ref{alg:sic_sda_commforth}. We present each step's condition (line~\ref{alg:sic_sda_commforth:line:condition}) and update (line~\ref{alg:sic_sda_commforth:line:update}) in Table~\ref{tab:sic_sda_comm_detail}. 

\begin{algorithm}[htp]
    \caption{\texttt{SIC-SDA.CommForth} \\ \hspace*{\fill} \((\bm{\Phi}_t= (\mathcal{A}_t, \mathcal{R}_t, \mathcal{L}_t, \bm{m}_t^l, \bm{m}_t^u))\)}
    \label{alg:sic_sda_commforth}
    \begin{algorithmic}[1]
        \For{Step \(s\in \{1,2,3,4,5\}\)} \Comment{1st-5th step}
            \For{\((k, j)\in \mathcal{K}_t \times [M_t]\)}
            \Statex \LeftComment{Leader's part}
                \If{rank \(i=1\) } 
                \If{\(s\)-th step's condition}\Comment{send signal}\label{alg:sic_sda_commforth:line:condition}
                    \State Play arm \(k\)-th active arm
                \Else\Comment{no update signal}
                    \State Play arm \((k+1)\Mod{K_t}\)-th active arm 
                \EndIf
            \Statex \LeftComment{Follower's part}
                \ElsIf{rank \(i = j\)} 
                    \State Play arm \(k\)-th active arm
                    \If{\(\1{a_{k,t} > 1} = 1\)} \Comment{detect signal}
                        \State \(s\)-th step's update \label{alg:sic_sda_commforth:line:update}
                    \EndIf
                \EndIf
            \EndFor
        \EndFor
    \end{algorithmic}
\end{algorithm}

\begin{table}[htp]
    \centering
    \caption{\texttt{SIC-SDA.CommBack}'s details}
    \label{tab:sic_sda_comm_detail}\tabcolsep=0.05cm
    \begin{tabular}{|c||c|c|}
        \hline
        Step & \(s\)-th step condition (leader) & \(s\)-th step update (followers) \\ \hline
        1st & \(k\in\mathcal{R}_t\) & \(\mathcal{K}_t \gets \mathcal{K}_t \setminus \{k\} \) \\
        2nd & \(k\in\mathcal{A}_t\) & \(\mathcal{K}_t \gets \mathcal{K}_t \cup \{k\} \) \\
        3rd & \(k=\mathcal{L}_t\) & \(\mathcal{L}_t\gets \mathcal{L}_t\) \\
        4th & \(m_{k,t}^l > m_{k,\text{pre}}^l\) & \(m_{k,t}^l \gets m_{k,t}^l+1\) \\
        5th & \(m_{k,t}^u < m_{k,\text{pre}}^u\) & \(m_{k,t}^u \gets m_{k,t}^u - 1\) \\
        \hline
    \end{tabular}
\end{table}


\subsection{Successive Accept and Reject by Leader}~\label{appsub:sic_sda_sar}

In this subsection, we illustrate how leader accepts and rejects arms in the shareable arm setting.
The successive accept and reject mechanism of \texttt{SIC-SDA} was made in~\cite{bubeck_multiple_2013} for top \(k\) arms identification and was utilized to design MMAB algorithms in~\cite{boursier_sic-mmab_2019}.
We further extend the idea to our \texttt{MMAB-SA} model where accepting and rejecting arms not only depends on arm's reward means, \textit{but also on theirs shareable resources which are unknown a prior.}

After leader receives followers' IE observations, 
the leader updates his empirical reward mean and check 
whether he can ensure (with high probability) that some optimal (sub-optimal) arms can be accepted (rejected). 
Denote \(g(k,j)\) as a special indicator function for distinguishing ``good'' arm \(k\) from ``bad'' arm \(j\) as follows
\[g(k,j)\coloneqq \1{\hat{\mu}_{k,t}- 3\sqrt{\frac{\log (T)}{2\tau_{k,t}}}\ge \hat{\mu}_{j,t}+  3\sqrt{\frac{\log (T)}{2\tau_{j,t}}}},\] 
where $g(k,j) =1$ implies that the arm \(k\)'s reward mean is higher than arm \(j\)'s with high confidence.

\noindent \textbf{Accept \(\mathcal{A}_t\):}
The leader accepts arm \(k\) if  the arm's resources are well estimated, i.e., \(k\in\mathcal{K}_t\setminus \mathcal{K}'_t\) and \(
    \sum_{j\in\mathcal{K}_t: g(k,j) = 0} m_{j,t}^u \le M_t
\), which means the total resources' upper confidence bounds 
of arms which cannot be distinguished as bad by arm \(k\) (i.e., \(g(k,j) = 0\)) is no greater than the number of active players \(M_t\).
Thus, the accept arm set \(\mathcal{A}_t\) can be set as:
\begin{equation}\label{eq:accept_set}
    \mathcal{A}_t \coloneqq \left\{k\in\mathcal{K}_t\setminus \mathcal{K}'_t:   \sum\nolimits_{j\in\mathcal{K}_t: g(k,j) = 0}  m_{j,t}^u \le M_t\right\}. 
\end{equation}

\noindent \textbf{Reject \(\mathcal{R}_t\):}
The leader rejects arm \(k\), if \(
    \sum_{j\in\mathcal{K}_t: g(j,k) = 1} m_{j,t}^l \ge M_t 
\). 
The inequality means that the total resources' lower confidence bounds of arms whose reward mean is higher than arm \(k\)'s with high confidence (i.e., \(g(j,k) = 1\)) is no less than \(M_t\).
Thus, the reject arm set \(\mathcal{R}_t\) can be expressed as:
\begin{equation}\label{eq:reject_set}
    \mathcal{R}_t \coloneqq \left\{k\in\mathcal{K}_t: \sum\nolimits_{j\in\mathcal{K}_t: g(j,k) = 1} m_{j,t}^l \ge M_t\right\}. 
\end{equation}

\noindent \textbf{Least favored arm \(\mathcal{L}_t\):}
An arm \(k\) is the empirical least favored among optimal arms if its empirical reward mean is higher than all other current active arms' with high confidence (i.e., the only optimal arm among them), and its resources' lower confidence bound is no less than \(M_t\), i.e.,
\begin{equation}\label{eq:marginal_arm}
    \mathcal{L}_t\! \coloneqq \!\left\{k\in\mathcal{K}_t \! : \!\! \sum\nolimits_{j\in\mathcal{K}_t} \! g(k,j)\! = \! K_t\! - \!1, m_{k,t}^l \!\ge\! M_t \right\}.
\end{equation}

The detailed successive accept and reject algorithm is presented in Algorithm~\ref{alg:sic_sda_accept_and_reject}.

\begin{algorithm}[htp]
    \caption{\texttt{SIC-SDA.AccRej}\\ \hspace*{\fill}
    \((\bm{\Lambda}_t= (\bm{S}^{\text{IE}}_t, \bm{S}^{\text{UE}}_t, \bm{\tau}_t, \bm{\iota}_t), (\mathcal{K}_t, M_t)\in\bm{\Psi}_t)\)}
    \label{alg:sic_sda_accept_and_reject}
    \begin{algorithmic}[1]
        \State \algorithmicensure \(\bm{\Phi}_t= (\mathcal{A}_t, \mathcal{R}_t, \mathcal{L}_t,\bm{m}_t^l, \bm{m}_t^u)\)
        \If{\(i=1\)} \Comment{\textit{Leader}}
            \State \(\hat{\mu}_{k,t}\gets {S}_{k,t}^{\text{IE}} / {\tau}_{k,t}, \hat{\nu}_t\gets {S}_{k,t}^{\text{UE}} / {\iota}_{k,t},\,\forall k\)
            \State Update \(\bm{m}^l_t, \bm{m}^u_t\) according to Eq.(\ref{eq:lower_bound_update})(\ref{eq:upper_bound_update})
            \State \multiline{Update \(\mathcal{A}_t\), \(\mathcal{R}_t\), \(\mathcal{L}_t\) by Eq.(\ref{eq:accept_set})(\ref{eq:reject_set})(\ref{eq:marginal_arm})}
        \EndIf
    \end{algorithmic}
\end{algorithm}

\subsection{Update Information for Each Player}~\label{appsub:sic_sda_update}

After the leader sends back his updated information to followers, all players update their status according to Algorithm~\ref{alg:sic_sda_update}.
It consists of three steps: (1) upon finding the least favored arm,
(line~\ref{alg:line:sic_sda_update_least_favored_arm}), 
all active players turn to exploit this arm,
(2) if there are accepted arms,
(line~\ref{alg:line:sic_sda_update_accepted_arm}), 
several players will turn to exploit these arm, and the number of players exploiting an arm is equal to the arm's resources capacity,
and (3) if  \(M_t > K_t\) (line~\ref{alg:line:sic_sda_update_b}), update the vector \(\bm{b}_t\) to record the player allocation in the next IE round. 

\begin{algorithm}[htp]
    \caption{\(\texttt{SIC-SDA.Update}\\ \hspace*{\fill}
    (i, (\mathcal{K}_t, M_t)\in\bm{\Psi}_t, \bm{\Phi}_t= (\mathcal{A}_t, \mathcal{R}_t, \mathcal{L}_t,\bm{m}_t^l, \bm{m}_t^u)\)}
    \label{alg:sic_sda_update}
    \begin{algorithmic}[1]
        \State \algorithmicensure \(\exparm, \bm{\Psi}_t=(\bm{b}_t, \mathcal{K}_t, M_t)\)
        \State \(M_t \gets M_t - \sum_{k\in\mathcal{A}_t}m_{k,t}^l\)
        \State \(\mathcal{K}_t \gets \mathcal{K}_t\setminus (\mathcal{A}_t \cup \mathcal{R}_t \cup \mathcal{L}_t)\)
        \Statex \LeftComment{Exploit the least favored arm}
        \If{\(\mathcal{L}_t \neq \emptyset\)}\label{alg:line:sic_sda_update_least_favored_arm}
            \State  \(\exparm\gets \mathcal{L}_t.\)
        \Statex \LeftComment{Exploit accepted arm}
        \ElsIf{\(i > M_t\)}\label{alg:line:sic_sda_update_accepted_arm}
            \State \multiline{\(\exparm\gets j\)-th accept arm where such that \(\sum_{n=1}^j m_{\sigma'(n),t}^l \ge i-M_t > \sum_{n=1}^{j-1} m_{\sigma'(n),t}^l\) where \(\sigma'(n)\) is the \(n\)-th accept arm.}
        \Else 
        \Statex \LeftComment{Update players' allocation in next IE round}
            \State \(\bm{b}_t\gets \bm{1}.\)  \label{alg:sic_sda_update:b_start}
            \If{\(M_t > K_t\)}\label{alg:line:sic_sda_update_b} 
                \State \(d\gets M_t - \abs{\mathcal{K}_t}.\)
                \ForAll{\(k \in \mathcal{K}_t\)}
                \While{\(d>0\) and \(b_{k,t} < m_{k,t}^l\)}
                    \State \(b_{k,t} \gets b_{k,t} + 1\); \(d\gets d - 1\).
                \EndWhile
                \EndFor
            \EndIf \label{alg:sic_sda_update:b_end}
        \EndIf
    \end{algorithmic}
\end{algorithm}
\section{Regret Analysis of \texttt{DPE-SDI} (Theorem~\ref{thm:dpe_sdi_regret})}\label{app:dpe_sdi_regret}
We separately bound the regret of each phase of \texttt{DPE-SDI}.
Denote \(R^{\text{init}}_{\texttt{DPE}}, R^{\text{explo}}_{\texttt{DPE}}, R^{\text{comm}}_{\texttt{DPE}}\) as the regret upper bounds of initialization phase, exploration-exploitation phase, and communication phase respectively. Then, the total regret \(\ERT\) is upper bounded as follows: 
\[
    \ERT \le R^{\text{init}}_{\texttt{DPE}} + R^{\text{explo}}_{\texttt{DPE}} + R^{\text{comm}}_{\texttt{DPE}},
\]
where \(R^{\text{init}}_{\texttt{DPE}}\)'s upper bound is provided in Eq.(\ref{eq:dpe_sdi_regret_init}), \(R^{\text{explo}}_{\texttt{DPE}}\)'s upper bound is provided in Eq.(\ref{eq:dpe_sdi_regret_explo}), and 
\(R^{\text{comm}}_{\texttt{DPE}}\)'s upper bound is provided in Eq.(\ref{eq:dpe_sdi_regret_comm}).

\subsection{Initialization Phase}\label{appsub:dpe_sdi_regret_init}

\begin{lemma}[{\cite{wang_optimal_2020}'s Lemma 1}]\label{lma:orthogoanlization}
    Let \(s\) denotes the number of rounds required to end the orthogonalization sub-phase. 
    Then \[
        \E[s] \le \frac{M(K-1)(K+1)}{K-M}.  
    \]
\end{lemma}

Note that the \texttt{DPE-SDI} applies the orthogonalization with the number of players \(M\) being known.
So we just need to replace the \(K\) in Lemma~\ref{lma:orthogoanlization} by \(M+1\), which show the total time slots in this sub-phase is at most \(M^3+2M^2+1\), where the additional \(1\) corresponds to ``Rally''. 
In each time slot of initialization phase, the cost is at most \(M\).
So, the total cost of initialization phase is upper bounded as follows. 
\begin{equation}
    \label{eq:dpe_sdi_regret_init}
    R^{\text{init}}_{\texttt{DPE}} \le M(M^3+2M^2+1)= M^4 + 2M^3 + M.    
\end{equation}

\subsection{Exploration-Exploitation Phase}\label{appsub:dpe_sdi_regret_explo}

We first state two useful lemmas as building blocks in this section's proof.  
\begin{lemma}[{\cite{wang_optimal_2020}'s Lemma 3}]
    \label{lma:empirical_mean}
    Let $k\in [K]$, and \(c>0\). Let \(H\) be a random set of rounds such that for all \(t,\, \{t\in H\}\in \mathcal{F}_{t-1}\). 
    Assume that there exists \((C_t)_{t\ge 0}\), a sequence of independent binary random variables such that for any \(t\ge 1\), 
    \(C_t\) is \(\mathcal{F}_t\)-measurable and \(\P[C_t=1]\ge c\). Further assume for any \(t\in H\), \(k\) is selected (\(\rho(t)=k\)) if \(C_t=1\). Then, 
    \[
        \sum_{t\ge 1}\P\left[\{ t\in H, \abs{\hat{\mu}_{k,t} - \mu_k} \ge \delta \}\right] \le 2c^{-1}(2c^{-1} + \delta^{-2}).
    \]
\end{lemma}

\begin{lemma}\label{lma:kl_ucb}
    Under the \emph{\texttt{DPE-SDI}} algorithm, for any arm \(k\in[K]\), we have 
    \[
        \sum_{t\ge 0} \P[u_{k,t} < \mu_k] \le 30M.
    \]
\end{lemma}

\begin{proof}
    In \texttt{DPE-SDI}, we update \(u_{k,t}\) once every \emph{exploration-exploitation phase} which is at most \(2M\) rounds. 
    Then, we have \[
        \sum_{t\ge 0} \P[u_{k,t} < \mu_k]  \le 2M\sum_{t'\ge 0} \P[u_{k,t'} < \mu_k],
    \]
    where \(t'\) represents the time slots when the KL-UCB index \(u_{k,t}\) is updated.
    Utilizing~\cite{combes_learning_2015}'s Lemma 6, we have \(\sum_{t'\ge 0} \P[u_{k,t'} < \mu_k] 
    \le 15\). Hence, we can conclude \(\sum_{t\ge 0} \P[u_{k,t} < \mu_k]  \le 30M.
        \)
\end{proof}
We first assume that the shareable resources capacities \(m_k\) of all arms are known, and bound the regret of \emph{exploration-exploitation phase} under this assumption, 
in which united explorations are omitted due to \(\mathcal{S}'_t=\emptyset\).
After that, we will show how to bound the regret of \emph{exploration-exploitation phase} under the unknown \(m_k\) case. 

Note that in \texttt{DPE-SDI}, all empirical reward means \(\hat{\mu}_{k,t}\) and KL-UCB indexes \(u_{k,t}\) are updated after every \textit{exploration-exploitation phase}. 
To simplify this kind of statistics update, we denote \(n(t)\) as the last update time slot for statistics used in time slot \(t\).
Denote $\hatbm{a}^*_t = \Oracle(\hat{\bm{\mu}}(t), \bm{m})$ as the best allocation based on all arms' empirical ``per-load'' reward means. 
Let $0<\delta< \min_{k=1}^{K-1}\frac{\mu_k - \mu_{k+1}}{2}$. We define several sets as follows,
\[
    \begin{split}
    \mathcal{A} \coloneqq & \{t\ge 1: \hatbm{a}^*_{n(t)} \neq \bm{a}^*\}, \\
    \mathcal{B} \coloneqq & \{t\ge 1: \exists k\in [K]\text{ s.t. } \hat{a}^{*}_{k,n(t)} > 0, \abs{\hat{\mu}_{k,n(t)} -\mu_k}\ge \delta\},\\
    \mathcal{C} \coloneqq & \{t\ge 1: \exists k\in[K]\text{ s.t. } a^*_k > 0, u_{k,n(t)}< \mu_k\},\\
    \mathcal{D} \coloneqq & \{t\ge 1: t \in \mathcal{A}\setminus (\mathcal{B}\cup\mathcal{C}), \exists k\in [K] \text{ s.t. } \\
    & \qquad\qquad\quad  a_k^*>0, \hat{a}_{k,n(t)}^*=0, \abs{\hat{\mu}_{k,n(t)} - \mu_k}\ge \delta\}.
    \end{split}
\]

\begin{lemma}\label{lma:kc_initial_bound_1}
    \(
        \mathcal{A}\cup\mathcal{B} \subseteq \mathcal{B}\cup\mathcal{C}\cup\mathcal{D}
    \) and thus \(
        \E[\abs{\mathcal{A}\cup\mathcal{B}}] \le  \E[\abs{\mathcal{B}}]+\E[\abs{\mathcal{C}}]+\E[\abs{\mathcal{D}}].
    \)
\end{lemma}
\begin{proof}[{Proof of Lemma~\ref{lma:kc_initial_bound_1}}]
    This proof is similar to~\cite{wang_optimal_2020} Lemma 5's. Denote \(t\in \mathcal{A}\setminus (\mathcal{B}\cup \mathcal{C})\). To prove the lemma, we need to show that \(t\in\mathcal{D}\). 
    Since \(t\notin \mathcal{B}\), for all \(k\in [K]\) such that \(\hat{a}_{k,n(t)}^*>0\), we have \begin{equation}\label{eq:kc_initial_bound_1:eq1}
        \abs{\hat{\mu}_{k,n(t)} - \mu_k} < \delta.
    \end{equation}
    Then, for \(t\in \mathcal{A}\setminus\mathcal{B}\), the optimal arm set \(\mathcal{S}^*\coloneqq\{k:a_k^* > 0\}\) is different from the empirical optimal arm set \(\mathcal{S}_{n(t)}\coloneqq\{k:\hat{a}_{k,n(t)}^* > 0\}\).
    Because \(t\notin \mathcal{B}\) implies that the order of empirical optimal arms is the same as the order of these arms' true reward means and thus \(\mathcal{S}^* = \mathcal{S}_{n(t)}\) is equivalent to \(\bm{a}^*=\hatbm{a}^*_{n(t)}\).
    So, there exists \(j\in[K], a^*_j>0, \hat{a}^*_{j,n(t)}=0\) such that
    \begin{equation}
        \label{eq:kc_initial_bound_1:eq2}
        \hat{\mu}_{j,t} < \hat{\mu}_{k,n(t)}\text{ for some }k, \hat{a}^*_{k,n(t)}>0, a^*_k=0.
    \end{equation}
    Combining (\ref{eq:kc_initial_bound_1:eq1}) and (\ref{eq:kc_initial_bound_1:eq2}) leads to \(
        \hat{\mu}_{j,t} < \hat{\mu}_{k,n(t)} \le \mu_k + \delta \le \mu_L - \delta \le \mu_j - \delta.
    \)
    The last two inequalities are due to our assumption that \(j \le L < k\) and \(\delta < \delta_0\). 
    It implies \(\abs{\hat{\mu}_{j,t} - \mu_j}\ge \delta\) and thus, \(t\in\mathcal{D}\). 
    Therefore, \(\mathcal{A}\cup\mathcal{B}\subseteq \mathcal{B}\cup\mathcal{C}\cup\mathcal{D}\).
\end{proof}


\begin{lemma}\label{lma:kc_initial_bound_2}
    \[
        \E[\abs{\mathcal{B}}] +\E[\abs{\mathcal{C}}] + \E[\abs{\mathcal{D}}] \le 6K^2 M^2 (4+\delta^{-2})
    \]
\end{lemma}

\begin{proof}[Proof of Lemma~\ref{lma:kc_initial_bound_2}]
    
\emph{To show $\E[\abs{\mathcal{B}}]\le 4KM(4+\delta^{-2})$. } Let \[
    \mathcal{B}_k\coloneqq \{t\ge 1: \hat{a}_{k,n(t)}^* > 0, \abs{\hat{\mu}_{k,n(t)} - \mu_k}\ge \delta\},
\]
we have $\mathcal{B} = \cup_{1\le k \le K}\mathcal{B}_k$.
Then, for any arm \(k\in [K]\), we decompose \(\mathcal{B}_k\) to two sets: \[
    \begin{split}
        \mathcal{B}_{k,1}=&\{t\in\mathcal{B}_k:l=k,i=1\},\\
        \mathcal{B}_{k,2}=&\mathcal{B}_k\setminus \mathcal{B}_{k,1}.
    \end{split}
\]
The \(\mathcal{B}_{k,1}\) set consists of time slots \(t\) in which it is the leader (\(i=1\))'s turn to play arm \(k\) (\(l =k\)).
In Lemma~\ref{lma:empirical_mean}, we set \(H = \{t\ge 1: l=k,i=1\}, C_t=\1{k_t = k} \) and thus $\P(C_t = 1) \ge \frac{1}{2}$. Then, we have $\E[\abs{\mathcal{B}_{k,1}}]\le 4(4+\delta^{-2})$.

To bound the cardinality of \(\mathcal{B}_{k,2}\), observe that whenever there is a \(t\in\mathcal{B}_{k,1}\), there is at most \(M-1\) other rounds \(t\in\mathcal{B}_k\), in which the leader rotates on other empirical optimal arms. Hence, \(\abs{\mathcal{B}_{k,2}} \le (M-1)\abs{\mathcal{B}_{k,1}}\). 

So, \(\E[\abs{\mathcal{B}}]\le \sum_{k=1}^K M\E [\abs{\mathcal{B}_{k,1}}] \le 4KM(4+\delta^{-2}).\)

\emph{To show $\E[\abs{\mathcal{C}}]\le 30M^2$. } Denote $\mathcal{C}_k \coloneqq \{t\ge 1: u_{k,n(t)} > \mu_k\}$ and we have $\mathcal{C} = \cup_{k=1}^L \mathcal{C}_k$. 
\[ 
        \E[\abs{\mathcal{C}}]\le \sum_{k=1}^L\E[\abs{\mathcal{C}_k}] \le L\cdot 30M \le 30M^2.
\]
where the second inequality holds by Lemma~\ref{lma:kl_ucb}.

\emph{To show $\E[\abs{\mathcal{D}}]\le 4K^2 M^2(4+\delta^{-2})$. } Denote $\mathcal{D}_k \coloneqq \{t\ge 1: t \in \mathcal{A}\setminus (\mathcal{B}\cup\mathcal{C}), a_k^* > 0, \hat{a}_{k,n(t)}^*=0, \abs{\hat{\mu}_{k,n(t)} -\mu_k}\ge \delta\}$. We have $\mathcal{D} =\cup_{k=1}^L\mathcal{D}_k$. 

As $t\notin\mathcal{C}$ we have $u_{k,n(t)} > \mu_k \ge \mu_L$. 
As $t\notin \mathcal{B}$ we know empirical optimal arms corresponding to $\hatbm{a}^*_t$ are ordered correctly and $\mu_{\mathcal{L}_t} + \delta > \\hat{\mu}_{\mathcal{L}_t,t}$. As $t\in \mathcal{A}$, we know the least favored arm's index $\mathcal{L}_t$ is greater than $L$, that is $\mu_L \ge \mu_{\mathcal{L}_t} + \delta$.
Together they lead to $u_{k,n(t)} \ge \\hat{\mu}_{\mathcal{L}_t,t}$, i.e., $k\in\mathcal{E}_t$. As the exploration arm is selected uniformly form $\mathcal{E}_t$, we know $\P(a_{k,n(t)} > 0) \ge \frac{1}{2K}$. 

Next, we decompose \(\mathcal{D}_k\) to sets \[
    \begin{split}
        \mathcal{D}_{k,1}   =& \{t\in\mathcal{D}_k: l = \mathcal{L}_t\},\\
        \mathcal{D}_{k,2} =& \mathcal{D}_k\setminus \mathcal{D}_{k,1}.
    \end{split}
    \]

When \(t\in \mathcal{D}_{k,1}\), i.e., the leader is playing the empirical least favored arm \(\mathcal{L}_t\), 
there is a probability of at least \(1/2K\) to explore the arm \(k\in\mathcal{E}_t\). 
In Lemma~\ref{lma:empirical_mean}, let $H=\{t\in\mathcal{A}\setminus (\mathcal{B}\cup\mathcal{C}), \hat{a}_{k,n(t)}^* = 0\}, C_t = \{a_{k,n(t)} > 0\}$, we have $\E[\abs{\mathcal{D}_{k,1}}] \le 4K(4K + \delta^{-2})\le 4K^2(4+\delta^{-2})$.

For the set \(\mathcal{D}_{k,2}\), as the leader rotates over all empirical optimal arms, \(\abs{\mathcal{D}_{k,2}}\) is no more than \(M-1\) times \(\abs{\mathcal{D}_{k,1}}\). In other words, \(\abs{\mathcal{D}_{k,2}}\le (M-1)\abs{\mathcal{D}_{k,1}}.\)

Summing up the above two cases, we obtain that $\E[\abs{\mathcal{D}}]
\le \E\left[\sum_{k=1}^L \abs{\mathcal{D}_{k}}\right]
\le \E\left[\sum_{k=1}^L (\abs{\mathcal{D}_{k,1}} + \abs{\mathcal{D}_{k,2}})\right] \le 4K^2M^2(4+\delta^{-2})$. 
\end{proof}

\begin{lemma}\label{lma:kc_explore_bound}
    Denote \(\mathcal{G}_k \coloneqq \{t\le T: t\notin \mathcal{A}\cup\mathcal{B}, \hatbm{a}^*_{n(t)} = \bm{a}^*, a_{k,t} > 0\}\) for arm $k$ that is not in the optimal action $\bm{a}^*$, i.e., $a_k^* = 0$. We have 
    \[
        \E[\abs{\mathcal{G}_k}] \le \frac{\log T + 4\log \log T}{\kl(\mu_k+\delta, \mu_L-\delta)} + 2(2+\delta^{-2}).
    \]
\end{lemma}

\begin{proof}
    Denote 
    $t_0 \coloneqq \frac{\log T + 4 \log\log T}{\kl(\mu_k+\delta, \mu_L -\delta)}$ and
    \[
        \begin{split}
            \mathcal{G}_{k,1}\coloneqq & \{t\in\mathcal{G}_k: \abs{\hat{\mu}_{k,n(t)} - \mu_k}\ge \delta\},\\
            \mathcal{G}_{k,2}\coloneqq & \left\{t \in \mathcal{G}_k: \sum_{\omega=1}^t \1{\omega\in\mathcal{G}_k} \le t_0  \right\}.
        \end{split}
    \]

    \emph{To show $\mathcal{G}_k \subseteq \mathcal{G}_{k,1}\cup \mathcal{G}_{k,2}$. } Let $t \in \mathcal{G}_k\setminus (\mathcal{G}_{k,1}\cup\mathcal{G}_{k,2})$. 

    As $k\in\mathcal{G}_k$ we have $u_{k,n(t)} \ge \hat{\mu}_{\mathcal{L}_t, t}$. As $t\notin \mathcal{A}\cup\mathcal{B}$ we have $\hat{\mu}_{\mathcal{L}_t,t} = \hat{\mu}_{L,t} \ge \mu_L - \delta$. As arm $k$ is suboptimal, we have $\mu_L - \delta \ge \mu_k + \delta$. As $k\notin \mathcal{G}_{k,1}$, we have $\mu_k+\delta \ge \hat{\mu}_{k,n(t)}$. Together, these lead to $u_{k,n(t)} > \hat{\mu}_{k,n(t)}$. 

    From $k\notin \mathcal{G}_{k,2}$, we have $t_0\le \sum_{\omega=1}^t \1{\omega  \in \mathcal{G}_k}\le \tau_{k,n(t)}$,
    where \(\tau_{k,n(t)}\) is the total number of times of IEs for arm \(k\). 
    \[
        \begin{split}
            t_0 \kl(\hat{\mu}_{k,n(t)}, \mu_L -\delta) \le & \tau_{k,n(t)} \kl(\hat{\mu}_{k,n(t)}, \mu_L - \delta) \\
            \le &\tau_{k,n(t)}\kl(\hat{\mu}_{k,n(t)}, u_{k,n(t)})\\
            \le & \log T + 4 \log\log T,
        \end{split}
    \]
    where the second inequality holds for $y\mapsto\kl(x,y)$ is increasing for $0<x<y<1$, and the last inequality holds for $u_{k,n(t)}$'s definition. 

    Substituting $t_0$ with its definition expression, we obtain $\kl(\hat{\mu}_{k,n(t)}, \mu_L - \delta) \le \kl(\mu_k + \delta, \mu_L - \delta)$. Note that $x\mapsto \kl(x,y)$ is decreasing for $0<x<y<1$, which further leads to $\hat{\mu}_{k,n(t)} \ge \mu_k + \delta$.
    This \emph{contradicts} the assumption that $t\notin\mathcal{G}_{k,1}$. So, $\mathcal{G}_k \subseteq \mathcal{G}_{k,1}\cup\mathcal{G}_{k,2}$.

    \emph{To bound $\E[\abs{\mathcal{G}_{k,1}}]$ and $\E[\abs{\mathcal{G}_{k,2}}]$. }
    In Lemma~\ref{lma:empirical_mean}, let $H=\{t\in \mathcal{G}_k\}, C_t = 1$, we have $\E[\abs{\mathcal{G}_{k,1}}]\le 2(2+\delta^{-2})$. For $\mathcal{G}_{k,2}$, we have \(\E[\abs{\mathcal{G}_{k,2}}]\le t_0.\) 
    Substituting \(\E[\abs{\mathcal{G}_{k,1}}]\) and \(\E[\abs{\mathcal{G}_{k,2}}]\) by their upper bound in the inequality \(\E[\abs{\mathcal{G}_k}] \le \E[\abs{\mathcal{G}_{k,1}}] + \E[\abs{\mathcal{G}_{k,2}}]\), we prove that: 
    \[
        \E[\abs{\mathcal{G}_k}] \le \frac{\log T + 4\log \log T}{\kl(\mu_k+\delta, \mu_L-\delta)} + 2(2+\delta^{-2}).
    \]
\end{proof}

So far, we have upper bounded the regret cost of \emph{exploration-exploitation phase} under the known resources capacity case. 
Next, we will show how to bound the phase's regret cost under the unknown resources capacity case. 
The definitions of sets \(\mathcal{A}, \mathcal{B}, \mathcal{C}, \mathcal{D}, \mathcal{G}_k\) all depend on the optimal profile \(\bm{a}^*=\Oracle(\bm{\mu}, \bm{m}).\)
We replace \(\bm{a}^*\) by \(\bm{a}^*(\bm{m}^l_t) = \Oracle(\bm{\mu}, \bm{m}^l_t)\) in their definitions and re-define these notations 
as \(\mathcal{A}(\bm{m}^l_t), \mathcal{B}(\bm{m}^l_t), \mathcal{C}(\bm{m}^l_t), \mathcal{D}(\bm{m}^l_t), \mathcal{G}_k(\bm{m}^l_t)\).
We note that the capacity lower confidence bound \(\bm{m}^l_t\) is updated after each \emph{exploration-exploitation phase}.
So, these sets are also time varying. 
Conditioned on the lower confidence bound \(\bm{m}^l_t\), the proof of Lemma~\ref{lma:kc_initial_bound_1}
for events \(\mathcal{A}(\bm{m}^l_t), \mathcal{B}(\bm{m}^l_t), \mathcal{C}(\bm{m}^l_t), \mathcal{D}(\bm{m}^l_t)\) still holds.
Lemma~\ref{lma:kc_initial_bound_2}'s result will then be:     \[
    \begin{split}
        &\E[\abs{\mathcal{B}(\bm{m}^l_t)}] +\E[\abs{\mathcal{C}(\bm{m}^l_t)}] + \E[\abs{\mathcal{D}(\bm{m}^l_t)}] \\
        \le & 2\times 6K^2 M^2 (4+\delta^{-2}),        
    \end{split}
\]
where the multiplier $2$ corresponds to two types of explorations: IE and UE.

In the worst case, one needs to learn the top $M$ arms' maximal resource capacities after all arm's reward means are well estimated, i.e., \(t\notin \mathcal{A}(\bm{m}^l_t)\cup\mathcal{B}(\bm{m}^l_t)\).
We adapt the Corollary~\ref{cor:m_sample_complexity} about the necessary sample size of UEs by substituting its \(\delta\) with \(2/T\) as follows.
\begin{corollary}\label{lma:ue_sample_size}
    For any arm \(k\) and \(T\ge \exp(49m_k^2/\mu_k^2)\), the inequality \(\P(\hat{m}_{k,t} = m_k)\ge 1-(2/T)\) holds if \[
        \tau_{k,t},\iota_{k,t} \ge \frac{49m_k^2}{\mu_k^2}\log T.
    \]
\end{corollary}
That implies when \(T\) is large, the \(\frac{49m_k^2}{\mu_k^2}\log T\) number of IEs and UEs of arm \(k\) can assure \texttt{DPE-SDI} learns the correct \(m_k\) with high confidence. 
So, in the worst case,
the cost of learning these resources capacities is upper bounded by $\sum_{k=1}^M\frac{49w_km_k^2\log (T)}{\mu_k^2}+2K^2$,     
where \(w_k\coloneqq f(\bm{a}^*) + \mu_1 - (m_k+1)\mu_k\) is the highest cost of one round of IE and UE. 

So, the regret cost of \emph{exploration-exploitation phase} is at most 
\begin{equation}
    \label{eq:dpe_sdi_regret_explo}
    \begin{split}
    R^{\text{explo}}_{\texttt{DPE}} \le & M\E[\abs{\mathcal{A}\cup \mathcal{B}}] + \sum_{k>L} (\mu_L - \mu_k)\E[\abs{G_k}]\\
    & +\sum_{k=1}^M\frac{49w_km_k^2\log (T)}{\mu_k^2}+2K^2\\
    \le & 6K^2 M^3 (4+\delta^{-2}) +  2K^2  +  \sum_{k=1}^M\frac{49w_km_k^2\log (T)}{\mu_k^2} \\
    & + \sum_{k=L+1}^K \frac{(\mu_L -\mu_k)(\log T + 4 \log(\log T))}{\kl(\mu_k+\delta,\mu_L-\delta)}.    
    \end{split}   
\end{equation}

\subsection{Communication Phase}\label{appsub:dpe_sdi_regret_comm}

Each \emph{communication phase} takes at most \(6K\) time slots. 
We note that in total, there are at most \(2\max\left\{ \abs{\mathcal{A}}, \frac{49m_M^2 \log(T)}{\mu_M^2} \right\}\) rounds of communication, 
where \(\abs{\mathcal{A}}\) corresponds to communication phases caused by the update of ``per-load'' reward mean estimate \(\hat{\mu}_{k,t}\) 
and \(\frac{49m_M^2 \log(T)}{\mu_M^2}\) corresponds to communication phases caused by the update of maximal resource capacity estimates \(\hat{m}_{k,t}\). 
So the regret cost of communication cost is at most 
\begin{equation}
    \label{eq:dpe_sdi_regret_comm}
    \begin{split}
    R^{\text{comm}}_{\texttt{DPE}} \le & 12KM\max\left\{ \abs{\mathcal{A}}, \frac{49m_M^2 \log(T)}{\mu_M^2} \right\} \\
    \le& 144K^3M^3(4+\delta^{-2}) + \frac{588KMm_M^2 \log(T)}{\mu_M^2},        
    \end{split}    
\end{equation}
where the last inequality holds for Lemma~\ref{lma:kc_initial_bound_2}.



\section{Regret Analysis of \texttt{SIC-SDA} (Theorem~\ref{thm:sic_sda_regret})}\label{app:sic_sda_regret}

We decompose the regret of the \texttt{SIC-SDA} algorithm to \(R^{\text{init}}_{\texttt{SIC}}, R^{\text{explo}}_{\texttt{SIC}}, R^{\text{comm}}_{\texttt{SIC}}\).
They correspond to 
the cost of initialization phase, exploration and exploitation phases, and communication phase respectively. 
Similar to Appendix~\ref{app:dpe_sdi_regret}, we have \[
    \ERT \le R^{\text{init}}_{\texttt{SIC}} + R^{\text{explo}}_{\texttt{SIC}} + R^{\text{comm}}_{\texttt{SIC}},
\]
where \(R^{\text{init}}_{\texttt{SIC}}\)'s upper bound is provided in Eq.(\ref{eq:sic_sda_regret_init}), \(R^{\text{explo}}_{\texttt{SIC}}\)'s upper bound is provided in Eq.(\ref{eq:sic_sda_regret_explo}), and 
\(R^{\text{comm}}_{\texttt{SIC}}\)'s upper bound is provided in Eq.(\ref{eq:sic_sda_regret_comm}).

\subsection{Initialization Phase}\label{appsub:sic_sda_regret_init}

\texttt{SIC-SDA}'s initialization phase is the same as~\cite{wang_optimal_2020}'s. 
So, their regret costs in this phase are also the same.
Lemma~\ref{lma:orthogoanlization} shows the phase's cost \(R^{\text{init}}_{\texttt{SIC}}\) is upper bounded as follows 
\begin{equation}
    \label{eq:sic_sda_regret_init}
    R^{\text{init}}_{\texttt{SIC}} \le M \left( \frac{K^2M}{K-M} + 2K \right).
\end{equation}

\subsection{Exploration Phase and Exploitation Phase}\label{appsub:sic_sda_regret_explo}

We first assume the shareable resources of each arm \(m_k\) are known a prior. 
That means \texttt{SIC-SDA} can omit united exploration (UE) since UE is only for estimating \(m_k\).
Denote \(T^{\text{explo}} \coloneqq \#\text{Explo}\) as the total number of time slots of all exploration and exploitation phases, and \(T_{k}^{\text{explo}}\) is the total number of time slots that arm \(k\) is pulled during the exploration or exploitation phase.
Excluding the time slots for communication and initialization and noticing that in the worst case, \(M\) players pulling the same arm may only generate one reward observation, the regret cost of exploration and exploitation phases can be upper bounded as follows~\cite{anantharam_asymptotically_1987}'s Eq.(5.1): 
\begin{equation}\label{eq:decompose_explo_regret}
    \begin{split}
     R^{\text{explo}}_{\texttt{SIC}} \le &  M\bigg(\sum_{k>L}(\mu_L - \mu_k) T_{(k)}^{\text{explo}} \\
     &+ \sum_{k\le L}(\mu_k - \mu_L)(T^{\text{explo}} - T^{\text{explo}}_{k})\bigg).           
    \end{split}
\end{equation}

The successive accept and reject process of \texttt{SIC-SDA} is the same as \texttt{SIC-MMAB} algorithm in~\cite{boursier_sic-mmab_2019}.
We can adapt their two lemmas as follows. 

\begin{lemma}[{Adapted from~\cite{boursier_sic-mmab_2019}'s Proposition 1}]\label{lma:elimination_sample_complexity}
    With probability \(1-O\left( \frac{K\log T}{T} \right)\), every optimal arm \(k\) is accepted after at most \(O\left( \frac{\log T}{(\mu_k -\mu_{L+1})^2} \right)\) pulls during exploration phases, and every sub-optimal arm \(k\) is rejected after at most \(O\left( \frac{\log T}{(\mu_L -\mu_k)^2} \right)\) pulls during exploration phases. 
\end{lemma}

\begin{lemma}[{Adapted from~\cite{boursier_sic-mmab_2019}'s Lemma 5}]
    With probability \(1-O\left( \frac{K\log T}{T} \right)\), the following holds simultaneously: 
    \begin{enumerate}
        \item For a sub-optimal arm \(k\), \((\mu_L - \mu_k)T_k^{\text{explo}}=O\left( \min \left\{ \frac{\log(T)}{\mu_L - \mu_k}, \sqrt{T\log (T)} \right\} \right)\).
        \item For an optimal arm \(k\), \(\sum_{k\le M} (\mu_k - \mu_L)(T^{\text{explo}} - T^{\text{explo}}_k) = O\left(\sum_{k>L}\min\left\{\frac{\log (T)}{\mu_L - \mu_k}, \sqrt{T\log(T)}\right\}\right)\).
    \end{enumerate}
\end{lemma}

    

The above two lemmas and Eq.(\ref{eq:decompose_explo_regret}) implies, when shareable resources \(m_k\) are known, we have \[
    R^\text{explo} \le O\left( \sum_{k>L} \min\left\{ \frac{M\log(T)}{\mu_L - \mu_k}, \sqrt{T\log (T)} \right\} \right)
\]
holds with probability \(1-O\left( \frac{K\log T}{T} \right)\).

We next consider the unknown shareable resources case.


To study the impact of unknown shareable resources in the successive accept and reject process, we consider the worst case:
except that the arms with indexes greater than \(M\) can be eliminated according to Lemma~\ref{lma:elimination_sample_complexity}, all other arms' shareable resources capacities are well estimated according to Corollary~\ref{lma:ue_sample_size}. 
So, the worst case's addition regret cost is at most\[
    \sum_{k=1}^M\frac{49w_km_k^2\log (T)}{\mu_k^2}+2K^2 = O\left( \sum_{k=1}^M\frac{m_k^2\log (T)}{\mu_k^2} \right),
\]
where \(w_k\coloneqq f(\bm{a}^*) + \mu_1 - (m_k+1)\mu_k\) is the highest cost of one round of IE and UE. 
Summing up with the known \(m_k\) case's result, we have 
\begin{equation}
    \label{eq:sic_sda_regret_explo}
    \begin{split}
        R^\text{explo} \le& O\left( \sum_{k>L} \min\left\{ \frac{M\log(T)}{\mu_L - \mu_k}, \sqrt{T\log (T)} \right\} \right)\\
        &+  O\left( \sum_{k=1}^M\frac{m_k^2\log (T)}{\mu_k^2} \right).
    \end{split}    
\end{equation}

\subsection{Communication Phase}\label{appsub:sic_sda_regret_comm}

The \texttt{SIC-SDA.CommBack} costs at most \(K^2 (p+1)\) per time, and the \texttt{SIC.SDA.CommForth} costs at most \(5K\) per time. 
Summing them up, we have \(\sum_{p=1}^N (K^2(p+1) + 5K) = O(K^2N^2)\), where \(N\) is the number of communication phases that is needed in \texttt{SIC-SDA}. 
By Lemma~\ref{lma:elimination_sample_complexity} for IE and Lemma~\ref{lma:ue_sample_size} for UE, \(N\) is at most \(   O\left( \log \left(\max \left\{\frac{\log (T)}{(\mu_L - \mu_{L+1})^2}, \frac{m_M^2\log(T)}{\mu_M^2} \right\}\right)\right)\).

So, the total cost of communication is upper bound as follows:
\begin{align}
    \label{eq:sic_sda_regret_comm}
        &R^\text{comm} \le \\
        &O\left( K^2 \log^2 \left(\max \left\{\frac{\log (T)}{(\mu_L - \mu_{L+1})^2}, \frac{m_M^2\log(T)}{\mu_M^2} \right\}\right)\right).\nonumber    
\end{align}


\section{Resource Capacity \(m_k\)'s Uniform Confidence Interval}~\label{app:update_ci}

\begin{lemma}[Confidence interval for shareable resources $m_k$]\label{lma:confidence_bound}
    Denote the function
    \(
        \phi(x,\delta) \coloneqq \sqrt{\left(1+\frac{1}{x}\right)\frac{\log(2\sqrt{x+1}/\delta)}{2x}}.
    \)
    When $\phi(\tau_{k,t},\delta) + \phi(\iota_{k,t},\delta) < \hat{\mu}_{k,t}$, the event
    \begin{align}
        &\bigg\{\!\forall \tau_{k,t}, \iota_{k,t} \!\in\!\mathbb{N}_+,
        m_k \!\in\! 
        \bigg[\frac{\hat{\nu}_{k,t}}{\hat{\mu}_{k,t} + \phi(\tau_{k,t},\delta) + \phi(\iota_{k,t},\delta)}, \nonumber \\
        &\qquad\qquad\qquad \frac{\hat{\nu}_{k,t}}{\hat{\mu}_{k,t} - \phi(\tau_{k,t},\delta) - \phi(\iota_{k,t},\delta)}\bigg] \bigg\}\label{eq:confidence_bound}
    \end{align}
    holds with a probability of at least $1-\delta$.
\end{lemma}
\begin{proof}[Proof of Lemma~\ref{lma:confidence_bound}]
    We apply the following Lemma~\ref{lma:uci_basic} to measure $\hat{\mu}_{k,t}$ and $\hat{\nu}_{k,t}$'s uncertainty.
    \begin{lemma}[{\cite{bourel_tightening_2020}'s Lemma 5}]\label{lma:uci_basic}
        Let $Y_1, \ldots, Y_t$ be a sequence of $t$ i.i.d. real-valued random variables with mean $\mu$, such that $Y_t- \mu$ is $\sigma$-sub-Gaussian. Let $\mu_t=\frac{1}{t}\sum_{s=1}^t Y_s$ be the empirical mean estimate. Then, for all $\sigma\in(0,1)$, it holds
        \begin{equation*}
            \mathbb{P}\left(\exists t \in \mathbb{N}, \abs{\mu_t - \mu} \ge \sigma \sqrt{\left(1+\frac{1}{t}\right)\frac{2\log(\sqrt{t+1}/\delta)}{t}}\right) \le \delta.
        \end{equation*}
    \end{lemma}

    Note that $X_k\in [0,1]$ is $1/2$-sub-Gaussian. Let $\sigma\gets 1/2$, $t\gets \tau_{k,t}$ and $\delta\gets \delta/2$ in Lemma~\ref{lma:uci_basic} we have
    \begin{equation*}
        \mathbb{P}(\exists \tau_{k,t}\in \mathbb{N}_+, \abs{\hat{\mu}_{k,t} - \mu_k} \ge \phi(\tau_{k,t},\delta) ) \le \delta/2,
    \end{equation*}
    where \begin{equation*}
        \phi(\tau_{k,t}, \delta) =\sqrt{\left(1+\frac{1}{\tau_{k,t}}\right)\frac{\log (2\sqrt{\tau_{k,t}+1}/\delta)}{2\tau_{k,t}}}
    \end{equation*}
    as we defined in the lemma.
    Then, the complementary event's probability is lower bounded as follows
    \begin{equation}\label{eq:mu_uci}
        \mathbb{P}(\forall \tau_{k,t}\in \mathbb{N}_+, \abs{\hat{\mu}_{k,t} - \mu_k} \le \phi(\tau_{k,t},\delta) ) \ge 1- \delta/2.
    \end{equation}
    Similarly, with a $1/m_k$ scaling for $\hat{\nu}_{k,t}$, we have
    \begin{equation}\label{eq:nu_uci}
        \mathbb{P}(\forall \iota_{k,t}\in \mathbb{N}_+, \abs{\hat{\nu}_{k,t} - m_k\mu_k} \le m_k\phi(\iota_{k,t},\delta) ) \ge 1- \delta/2.
    \end{equation}

    The confidence intervals of Eq.(\ref{eq:mu_uci}) and Eq.(\ref{eq:nu_uci}) are as follows
    \begin{align}
        {\mu}_{k} &\in \left[\hat{\mu}_{k,t} - \phi(\tau_{k,t},\delta), \hat{\mu}_{k,t} + \phi(\tau_{k,t},\delta)\right],\label{eq:mu_uci_interval}\\
        m_k\mu_k &\in \left[\hat{\nu}_{k,t} - m_k\phi(\iota_{k,t},\delta), \hat{\nu}_{k,t} + m_k\phi(\iota_{k,t},\delta)\right].  \label{eq:nu_uci_interval}          
    \end{align}
    Then we rearrange term of (\ref{eq:nu_uci_interval})'s right hand side and get
     \[
        m_k \le \frac{\hat{\nu}_{k,t}}{\mu_k - \phi(\iota_{k,t}, \delta)}\le \frac{\hat{\nu}_{k,t}}{\hat{\mu}_{k,t}-\phi(\tau_{k,t},\delta)-\phi(\iota_{k,t},\delta)},
    \]
    where the second inequality is obtained by substituting \(\mu_k\) with its lower bound in (\ref{eq:mu_uci_interval}).
    Similarly, we can also obtain the lower confidence bound of \(m_k\) via the same technique. 
    Both together leads to the following interval 
    \begin{equation*}\begin{split}
        &m_k\in\\
        &\left[\frac{\hat{\nu}_{k,t}}{\hat{\mu}_{k,t}+\phi(\tau_{k,t},\delta) + \phi(\iota_{k,t},\delta)}, 
        \frac{\hat{\nu}_{k,t}}{\hat{\mu}_{k,t}-\phi(\tau_{k,t},\delta)-\phi(\iota_{k,t},\delta)}\right]. 
    \end{split}
    \end{equation*}

    Finally, applying the union bound to Eq.(\ref{eq:mu_uci}) and Eq.(\ref{eq:nu_uci}), we have
    \begin{equation*}
        \begin{split}        
            \mathbb{P}\Bigg(\forall \tau_{k,t}, \iota_{k,t} \in \mathbb{N}_+, m_k\in
            \Bigg[\frac{\hat{\nu}_{k,t}}{\hat{\mu}_{k,t}+\phi(\tau_{k,t},\delta) + \phi(\iota_{k,t},\delta)}, \\
            \frac{\hat{\nu}_{k,t}}{\hat{\mu}_{k,t}-\phi(\tau_{k,t},\delta)-\phi(\iota_{k,t},\delta)}\Bigg]\Bigg) 
            \ge 1-\delta.
        \end{split}
    \end{equation*}
\end{proof}


From Lemma~\ref{lma:confidence_bound}, we can calculate the shareable resources \(m_k\)'s lower and upper confidence bounds as follows:
\begin{align}
    & m_{k,t}^l = \\
    &\max\left\{m_{k,t-1}^l, \ceil{\frac{\hat{\nu}_{k,t}}{\hat{\mu}_{k,t} + \phi(\tau_{k,t}, \delta) + \phi(\iota_{k,t}, \delta)}}\right\},\nonumber\\
    &m_{k,t}^u = \\
    & \min\left\{m_{k,t-1}^u, \floor{\frac{\hat{\nu}_{k,t}}{\hat{\mu}_{k,t} - \phi(\tau_{k,t}, \delta) - \phi(\iota_{k,t}, \delta)}}\right\}.\nonumber
\end{align}

\begin{corollary}[Sample Complexity]\label{cor:m_sample_complexity}
    For any arm $k$ and \(0<\delta \le 2\exp(-49m_k^2 / \mu_k^2)\),
    the inequality
    $
        \mathbb{P}({\hat{m}_{k,t}} = m_k) \ge 1 - \delta
    $ holds if \[
        \tau_{k,t}, \iota_{k,t} \geq
        \frac{49m_k^2\log (2/\delta)}{\mu_k^2}.
    \]
\end{corollary}

\begin{proof}[Proof of Corollary~\ref{cor:m_sample_complexity}]
    From Lemma~\ref{lma:confidence_bound} and $m_k\in\mathbb{N}_+$, we learn $m_k$ before the interval width is less than $1$. That is, 
    \begin{equation*}
        \begin{split}
            \frac{\hat{\nu}_{k,t}}{\hat{\mu}_{k,t}-\phi(\tau_{k,t},\delta)-\phi(\iota_{k,t},\delta)}
            -
            \frac{\hat{\nu}_{k,t}}{\hat{\mu}_{k,t}+\phi(\tau_{k,t},\delta) + \phi(\iota_{k,t},\delta)}\\
            \le 1.            
        \end{split}
    \end{equation*}
    It reduces to 
    \[
    \begin{split}
        (\phi(\tau_{k,t},\delta) + \phi(\iota_{k,t},\delta))^2 +2\hat{\nu}_{k,t}(\phi(\tau_{k,t},\delta)+\phi(\iota_{k,t},\delta)) -\hat{\mu}_{k,t}^2 \\
        \le 0.        
    \end{split}
    \]
    Replace $\hat{\nu}_{k,t}$ and $\hat{\mu}_{k,t}$ with their confidence lower and upper bounds respectively, we further have 
    \[
        \begin{split}
            &(\phi(\tau_{k,t},\delta) + \phi(\iota_{k,t},\delta))^2  -(\mu_k + \phi(\tau_{k,t},\delta))^2 \\
            &+2(m_k(\mu_k + \phi(\iota_{k,t},\delta)))(\phi(\tau_{k,t},\delta)+\phi(\iota_{k,t},\delta))\le 0.            
        \end{split}
    \] 

    Rearrange the terms, it becomes 
    \[ 
        \begin{split}
            ((2m_k+2)\phi(\tau_{k,t},\delta) + (2m_k+1)\phi(\iota_{k,t},\delta) -\mu_k )\\
            \times (\mu_k + \phi(\iota_{k,t}))\le 0.        
        \end{split}
    \]

    As the term \((\mu_k + \phi(\iota_{k,t}))\) in LHS is positive, we finally have 
    \[
        (2m_k + 2)\phi(\tau_{k,t},\delta) + (2m_k + 1)\phi(\iota_{k,t},\delta) \le \mu_k.
    \]

    One solution is to require both $\phi(\tau_{k,t},\delta)$ and $\phi(\iota_{k,t},\delta)$ no greater than  $\frac{\mu_k}{7m_k}$. 
    Solving these, we have
    \begin{equation*}
        \tau_{k,t},\iota_{k,t} \ge \frac{49m_k^2\log (2/\delta)}{\mu_k^2},
    \end{equation*}
    where \(\delta \le 2\exp(-49m_k^2 / \mu_k^2)\).
\end{proof}

\end{document}